\documentclass[letterpaper, 10pt, journal]{IEEEtran}

\long\def \omitit#1{}
\usepackage{graphicx}
\usepackage{color}
\usepackage[linesnumbered]{algorithm2e}
\usepackage{amsmath} 
\usepackage{amssymb}  
\usepackage{amsthm}
\usepackage{multirow}
\usepackage{hhline}
\usepackage{epstopdf}
\usepackage{gensymb}

\newtheorem{theorem}{Theorem}[section]

\newtheorem{lemma}[theorem]{Lemma}
\newtheorem{definition}{Definition}

\newcommand{\union}{\cup}

\title{\LARGE \bf
Integrated Task and Motion Planning for Multiple Robots under Path and Communication Uncertainties
}

\author{Bradley Woosley and Prithviraj Dasgupta
\thanks{*This work was partially supported as part of the COMRADES project supported by the U.S. Office of Naval Research and by a GRACA grant from the University of Nebraska Omaha}
\thanks{B. Woosley is a graduate student and P. Dasgupta is a Professor with the Computer Science Department,
        University of Nebraska, Omaha, USA.
        {\tt\small \{bwoosley, pdasgupta\}@unomaha.edu}}%
}

\begin{document}

\maketitle
\thispagestyle{empty}
\pagestyle{empty}

\omitit{
Outline: 
1. Introduction - explain what the problem is, why this problem is relevant to robots, which application domains of robotics will benefit by solving this problem, what is your solution approach and what advantage your solution approach will have over other existing solutions
2. Related Work - summarize related work on STAMP - include work by Kavraki's group, Loibl's paper, Marthi's paper (that used logic for the STAMP problem)
3. Model - Explain the mathematical framework - representation of the tasks as a connected graph and finding the best traversal of the graph by following the best policy using MDPs and HMMs, mention also about the lower level path planner that is used
3.1. Algorithms
4 Analytical Results (if any)
5. Experimental Results - compare with MRTA-RTPP, another algorithm possibly from one of Kavraki's papers
}

\begin{abstract}
We consider a problem called task ordering with path uncertainty (TOP-U) where multiple robots are provided with a set of task locations to visit in a bounded environment, but the length of the path between a pair of task locations is initially known only coarsely by the robots. The objective of the robots is to find the order of tasks that reduces the path length (or, energy expended) to visit the task locations in such a scenario. To solve this problem, we propose an abstraction called a task reachability graph (TRG) that integrates the task ordering with the path planning by the robots. The TRG is updated dynamically based on inter-task path costs calculated using a sampling-based motion planner, and, a Hidden Markov Model (HMM)-based technique that calculates the belief in the current path costs based on the environment perceived by the robot's sensors and task completion information received from other robots. We then describe a Markov Decision Process (MDP)-based algorithm that can select the paths that reduce the overall path length to visit the task locations and a coordination algorithm that resolves path conflicts between robots. We have shown analytically that our task selection algorithm finds the lowest cost path returned by the motion planner, and, that our proposed coordination algorithm is deadlock free. We have also evaluated our algorithm on simulated Corobot robots within different environments while varying the number of task locations, obstacle geometries and number of robots, as well as on physical Corobot robots. Our results show that the TRG-based approach can perform considerably better in planning and locomotion times, and number of re-plans, while traveling almost-similar distances as compared to a closest first, no uncertainty (CFNU) task selection algorithm.
\end{abstract}

\section{Introduction}
Multi-robot task planning and path planning are important problems in multi-robot systems when robots have to perform tasks at different locations within an environment. The problem is encountered in many applications of multi-robot systems such as automated surveillance~\cite{Zlot06}, robotic demining~\cite{Lenagh15}, and automated inspection of engineering structures~\cite{Rutishauser09}. As a motivating example, we consider a scenario for performing standoff detection of explosives or landmines using autonomous robots where multiple robots are provided with a coarse map containing locations of objects of interest. The robots are required to autonomously plan their paths to get in proximity of each object of interest so that they can analyse the object with their detection sensors. For realizing this, the main computational problem is to calculate a suitable task plan or ordering among the tasks for each robot so that a performance metric, such as the energy expended or the time required by robots to perform the tasks gets reduced. Researchers have proposed Multi-robot Task Allocation (MRTA) techniques ~\cite{Korsah13} as well as multi-robot path planning techniques~\cite{Wagner15} to address this problem. However, on one hand, most MRTA techniques assume that the costs or distances between the task locations are fixed and known to all the robots as soon as they become aware of the task. This criterion might not be valid if the robots have a coarse map of the environment and the path cost between tasks can change dynamically as the robots discover obstacles in the environment, or if due to communication constraints, the delivery of a task completed message is delayed. On the other hand, path planning techniques account for dynamically discovered obstacles but they focus mainly on finding collision and conflict-free paths for robots and do not adjust the ordering between the waypoints or task locations that are being visited by the robots. Keeping a fixed order between tasks might result in unnecessary longer paths to complete the task schedule, especially when a dynamically updated path around an obstacle could induce a shorter task schedule.

\begin{figure}
\includegraphics[width=3.5in]{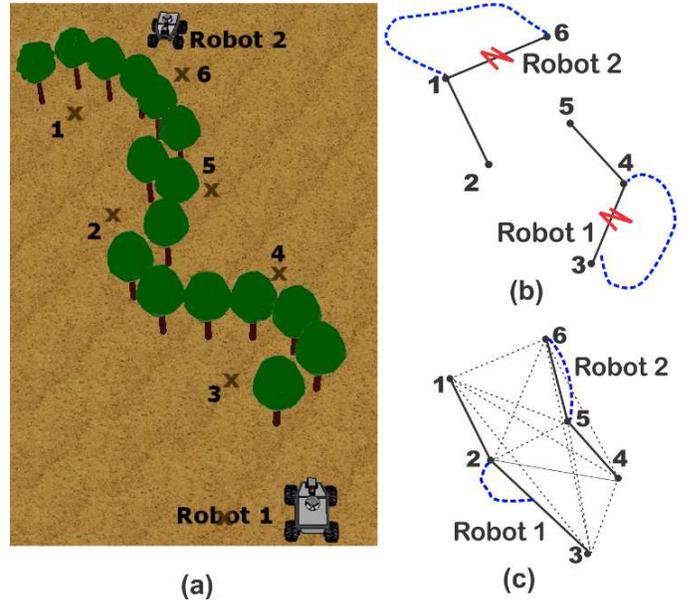}
\caption{{\small{(a) Scenario showing two robots and six tasks, each task needs $1$ robot to get completed, (b) Tasks selected by robots using CFNU algorithm; red-marked edges are unnavigable, blue dashed edges show re-calculated, collision-free paths, (c) Our proposed algorithm calculates a different task schedule for each robot that results in lower path costs by including uncertainty in path cost and availability in real-time into the task schedule calculation. Dotted edges show the task reachability graph (TRG) edges not followed by each robot due to higher expected costs; dashed edges are collision-free paths calculcated by the motion planner.}}}
\label{fig_topu}
\end{figure}

To address the above problem, it would make sense to investigate techniques that have a closer integration between the task and motion planning operations of robots. Researchers have proposed the Simultaneous Task and Motion Planning (STAMP) problem to investigate this problem in the context of robotic manipulation~\cite{Sucan-ICRA-2012}. Our work advances this direction of research by proposing a framework called Task Ordering under Path Uncertainty (TOP-U) to address the STAMP problem in the context of a search and exploration scenario using wheeled, mobile robots. An example scenario shown in Figure~\ref{fig_topu} illustrates an example of the reduced time taken and distance traveled when  robots use our proposed algorithm, in contrast with Closest First No Uncertainty (CFNU) algorithm used for task selection, that does not consider uncertainty in path costs and availabilities in its calculations. The main contributions of our paper are the following: We present a formalization of the problem called task ordering with path uncertainty (TOP-U), where multiple robots have to visits tasks whose locations are uncertain due to the presence of obstacles in the environment, while reducing the distances traveled between the tasks. We propose a data structure called a task reachability graph (TRG) that is used to model the problem and a Markov Decision Process (MDP)-based algorithm that each robot uses to dynamically calculate its task schedule in real-time using the TRG. We also propose a distributed coordination algorithm for resolving deadlock scenarios due to path conflicts between multiple robots using our algorithm. We have proved analytically that our task scheduling algorithm is optimal and the coordination algorithm is deadlock free. We also provide extensive experimental results on simulated and physical Coroware Corbot robots, with different number of robots and tasks within environments with different obstacles geometries and task distributions. Our results show that our proposed TRG-based approach could perform up to $51\%$ better in planning and locomotion times with $20\%$ fewer replans, while traveling similar distances as compared to a closest first, no uncertainty (CFNU) selection algorithm. The rest of our paper is structured as follows: in the next section, we discuss existing research on MRTA and motion planning techniques. In Section~\ref{sec_topu}, we formalize the TOP-U problem and describe the robot task scheduling and multi-robot coordination algorithms. Sections~\ref{sec_proofs} and~\ref{sec_expts} describe our analytical and experimental results and finally we conclude.

\section{Related Work}
\label{sec_relatedwork}

Motion and task planning have been important problems in robotics. Several approaches for solving them have been proposed in literature over the past two decades, although these problems have largely been treated separately. In motion planning the objective is to find a collision free path for a robot so that it can navigate within its environment~\cite{Choset05}. Sampling-based motion planners like probabilistic roadmap (PRM)~\cite{Bohlin-ICRA-2000} and Rapidly-exploring Random Trees (RRT)~\cite{Lavalle98} have been used widely for motion planning. Recently, researchers have proposed extensions to these techniques by using methods to reduce the time required to calculate the paths and address the problem of moving through narrow passages~\cite{Gammell14,Hauser15,Voss15}, and handling uncertainty in obstacle locations~\cite{Kneebone09,Missiuro-ICRA-2006}. Researchers have also investigated the problem of coordinating the paths between multiple robots~\cite{Desaraju12,Wagner15} where robots exchange their individual paths with each other and mutually exclusive paths are calculated for each robot in the robots' joint configuration space. To reduce the complexity of planning in the joint configuration space, a lightweight protocol was proposed in ~\cite{Luna11} where robots iteratively make way for one robot at a time to reach its goal until a consistent set of maneuvers have been determined for all robots to reach their goal. We use a complementary approach in this paper that allows robots to mostly calculate path plans individually in their local configuration space but if a set of robots get within close proximity of each other they use a conflict resolution algorithm to find collision free paths. Path planning in dynamic environments where the cost between the source and goal locations can change abruptly was addressed in~\cite{Marthi-RSS-2012, Loibl-ICRA-2013} using Markov Decision Processes (MDPs). In our proposed approach, path costs are also dynamically updated using the robots' sensor information, and the updated path costs are used immediately to recalculate the task schedule using an MDP to allow for switching between tasks to reduce the cost of the total path length to visit all the tasks.

The problem of finding a suitable ordering of operations or tasks to perform by multiple robots has been researched as the Multi-Robot Task Allocation (MRTA) problem~\cite{Gerkey-IJRR-2004}, excellent reviews of MRTA are available in~\cite{Dias06,Korsah13}. Most of the approaches focus mainly on finding a suitable ordering of tasks while assuming appropriate robot motion planning techniques. Recently, researchers have addressed more tightly coupled task and motion planning under the simultaneous task and motion planning (STAMP) problem. The proposed solution techniques  combine symbolic task planning with control based techniques~\cite{Kaelbling13, Sucan-ICRA-2012, Wolfe-ICAPS-2010} for a mobile manipulation problem where task interdependencies form a critical aspect and reasoning using symbolic task planning is critical to determine the task precedence. In contrast, for our setting, reducing the cost of the task schedule is more critical than the order of tasks, and our algorithm uses probabilistic methods to quickly incorporate robots' perceptions about the environment into its plan. 

\section{Task Ordering with Path Uncertainty}
\label{sec_topu}
We consider a set of wheeled, mobile robots, $R$, deployed within an environment. Robots are capable of localizing themselves within the environment and can also communicate wirelessly with each other. The environment contains a set of tasks, $T$. Robots have to visit the locations of tasks to perform operations required to complete the tasks. Each task can require visits by one of more robots to get completed; the information about how many robots are required to complete a task is provided {\em a priori} to the robots. We consider tasks that are loosely coupled and all robots required to complete a task do not necessarily need to visit the task's location at the same time. Each robot is initially aware of the locations of the tasks, but does not know the exact paths between the tasks{\footnote{In the rest of the paper, we have referred to task locations as tasks for legibility.}} or the obstacles along those paths. To represent this path uncertainty, each robot uses a task reachability graph (TRG), a fully connected graph with task locations as its vertices. Formally, a TRG is denoted by $TRG=(V, E, C, P, t)$ where:

\begin{itemize}
\item $V^{(t)}=\{v_i^{(t)} \cup v_{curr}\}$ is the vertex set and $v_{curr}$ is the robot's current location. Each $v_i^{(t)}$ corresponds to a task location the robot is aware of at time $t$ 
\item $E^{(t)}=\{e_{ij}^{(t)}: e_{ij}^{(t)} =(v_i^{(t)}, v_j^{(t)})\}$ is the edge set connecting the vertices in the TRG
\item $C^{(t)}=\{c_{ij}^{(t)}\}$ is the expected distance or cost expended by a robot to traverse the path underlying edge $e_{ij}$. $c_{curr,i}$  denotes the expected cost from the robot's current location ($v_{curr}$) to $v_i$
\item $P^{(t)}=\{p_{ij}^{(t)}\}$ is the probability that edge $e_{ij}^{(t)}$ is not available
\end{itemize}

Owing to path uncertainties between task locations (TRG vertices), $C^{(t)}$ and $P^{(t)}$ are estimated from perceived sensor data and they get updated by the robot as it discovers obstacles and task availabilities while navigating between tasks. Let $S: V \rightarrow V$ denote a function that returns an ordering over the set of tasks. Each robot maintains its own TRG and plans its path using its TRG. The problem facing each robot to find a suitable order for visiting the tasks is specified by the Task Ordering under Path Uncertainty (TOP-U) problem below:

{\bf TOP-U Problem.} Given $TRG=(V, E, C, P, t)$ representing the set of tasks, inter-task costs and task availabilities at time $t$, determine a schedule $S^*(V)^{(t)}$ that induces an ordering $(v^1, v^2, v^3...)$ over the tasks, given by:
\[ S^*(V)^{(t)} = \underset{S(V)^{(t)}}{\arg\min} \displaystyle{\sum_{(v_i,v_j)} \in S(V)^{(t)}} (1-p_{ij}^{(t)}) c_{ij}^{(t)}\] 
subject to:
\begin{eqnarray}
0 \leq p_{ij}^{(t)} \leq 1 \nonumber \\
\displaystyle{\sum_{(v_i,v_j) \in S(V)^{(t)}} (1-p_{ij}^{(t)})c_{ij}^{(t)} \le {\cal B}^{(t)}} 
\label{eq:topu}
\end{eqnarray}

where ${\cal B}^{(t)}$ is the battery available to a robot at current time $t$. Note that $S^*(V)^{(t)}$ represents the path through the TRG with the minimum expected cost, weighted with availability. The second constraint above ensures that the robot is able to complete this path with its currently available battery. Note that $\{V\}$ can change dynamically for a robot as tasks can get completed by other robots. An instance of the TOP-U problem corresponds to the well-known traveling salesman problem (TSP)~\cite{Cormen09}. However, solving the conventional TSP might not guarantee an optimal path as edge costs ($c_{ij}$) could change dynamically as robots discover previously unknown obstacles while traveling between tasks, while edge availabilities ($p_{ij}$) could change dynamically because some tasks got completed by other robots. To address the dynamic nature of the problem, we propose a Hidden Markov Model (HMM)-based method to update the edge availabilities, and then use the updated information within an MDP to find the desired ordering of the TRG vertices to solve the TOP-U problem. In the rest of the paper, for legibility, we have omitted the time notation from the TRG parameters, assuming it to be understood from context.

\subsection{Dynamically Updating Edge Cost and Availability}
\label{sec:edgecost}
{\bf Edge Cost Update.} TRG edge costs correspond to the distance that the robot requires to travel to reach from one TRG vertex to another. Each edge cost is initialized to the Euclidean distance between the pair of TRG vertices forming the edge. However, if there exist previously unknown obstacles in the path between a pair of TRG vertices, then the distance the robot travels might exceed the Euclidean distance between the vertices. To accommodate the path distance uncertainty, the robot uses a probabilistic roadmap(PRM)-based path planner~\cite{Missiuro-ICRA-2006} to dynamically update the expected edge cost. The PRM planner works by first generating a set of sampled points ${\cal R}$ from the robot's configuration space. It then uses the available information about obstacles perceived by the robot from its current location $v_{curr}$ to determine path segments that are close to obstacles and might result in collision with high probability; such segments are associated with a high penalty value. Following~\cite{Missiuro-ICRA-2006}, the TRG edge cost calculations are given by the following steps.

\begin{enumerate}
\item Calculate the cost of each path segment that connects any two sampled points $(\rho_1, \rho_2) \in Q_{free} \subseteq \Re^2$ ($Q_{free}$ is the free space in the environment) as:
\begin{eqnarray}
cost{\rho_1,\rho_2} =  p^{coll}_{\rho_1, \rho_2}penalty + (1-p^{coll}_{\rho_1, \rho_2})dist(\rho_1, \rho_2), \nonumber 
\end{eqnarray}
where $p^{coll}_{\rho_1, \rho_2}$ the probability of collision of segment $(\rho_1, \rho_2)$ based on its distance to perceived obstacles, $penalty$ is an arbitrary large number used to discourage paths that have a high probability of collision and $dist(\rho_1,\rho_2)$ is the Euclidean distance between $\rho_1$ and $\rho_2$
\item Calculate the physical path $\rho_{ij}$ corresponding to TRG edge $e_{ij}=(v_i, v_j)$ as a sequence of path segments $\rho = (\rho_1, \rho_2)_{start}...(\rho_1,\rho_2)_{end}$, given by:
\begin{eqnarray}
\rho_{ij} = \underset{\rho}{\arg \min} \sum_{(\rho_1, \rho_2) \in \rho} cost{\rho_1,\rho_2} \nonumber\\
\mbox{s. t.:}\quad \rho_{1_{start}} = v_i, \rho_{2_{end}} = v_j \nonumber
\end{eqnarray}
\item Calculate the expected cost $c_{ij}$ for TRG edge $e_{ij}$ as the sum of the costs of path segments in path $\rho_{ij}$ calculated in step $2$ above, as: 
\begin{equation}
c_{ij} =  \sum_{(\rho_1, \rho_2) \in \rho_{ij}} cost{\rho_1,\rho_2}
\label{eq:prm-cost}
\end{equation}

\end{enumerate}

\begin{figure}
\centering
\includegraphics[width=\linewidth]{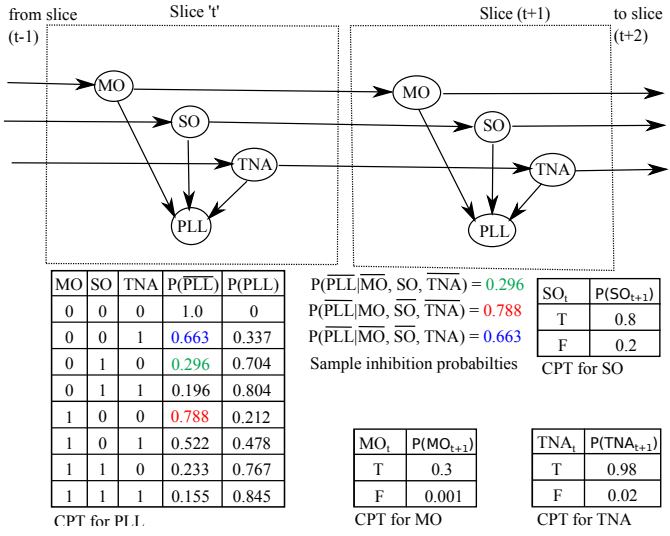}
\caption{Temporal Bayesian network used in the HMM for determining the suitability of path length to a task.}
\label{fig:dbn}
\end{figure}

{\bf Edge Availability Update.} In our scenario, tasks are completed in a distributed manner by different robots and a task (or, TRG vertex for a robot) might get completed by other robots before the robot reaches it. When a task is completed, the last robot visiting the task broadcasts a task completion message to all other robots. Each robot then needs to remove the TRG vertex for the task from its TRG. Because message communication in unstructured environments might be unreliable, task completion message might be lost due to noisy communication, or, the robot broadcasting the message might be outside the communication range of some robots. To handle this uncertainty, it would be useful if the robot could infer whether the task was still available or not, from information related to the task's availability that it can directly observe. For our problem, we assume that this observable information related to task availability is the distance or path length remaining to reach the task - if the path length is very large, it could be due to the task becoming unavailable{\footnote{Note that when a TRG edge is removed, it can be looked upon as the edge length becoming infinitely large.}}. One caveat to using the path length as an indicator of task availability is that it is also affected by obstacles along the path; it changes dynamically as the robot encounters obstacles while going towards the task. The problem facing the robot then is to observe the path length values over the recent past and infer from it whether the task is still available. 

To model this inference problem, each robot uses a Hidden Markov Model (HMM)~\cite{Russell09}; one HMM, $HMM_{ij}$, is used to update the availability of each TRG edge $e_{ij}$. The crucial HMM variable is Path Length Long ($PLL$) that evolves temporally as the robot encounters static obstacles ($SO$) and mobile obstacles ($MO$) on its range sensor, or, receives a task completion message called task not available ($TNA$), as shown in Figure \ref{fig:dbn}.  Variables $SO, MO$ and $TNA$ are binary-valued and they too evolve with time as the robot moves towards the task and encounters obstacles, or receives task completion message. The temporal transition model is given in Figure~\ref{fig:dbn} via the arrows moving between dashed boxes, the sample temporal probabilities of these variables are also provided. The dependencies between these variables affecting $PLL$ are captured in each slice of the HMM as shown inside the dashed boxes in Figure~\ref{fig:dbn}. Because each of the variables affecting $PLL$ - static obstacles, mobile obstacles and task not available - do not affect each other and can be considered as independent of each other, their probabilistic effect on $PLL$ can be combined relatively easily from the individual inhibition probabilities for these variables using a noisy-OR model. An example noisy-OR based probability calculation for $PLL$ is shown alongside Figure \ref{fig:dbn}.  We assume that the environment has a $98\%$ communication success rate, or $2\%$ communication failure rate, leading to the probability value given in Figure~\ref{fig:dbn}.

To solve the problem of calculating the probability of a task being still available from $PLL$ values, the robot first calculates the observed value of $PLL$ variable for the current time step. For $HMM_{ij}$, the observed value of variable $PLL_{ij}$ for current time step $t$ is determined by assuming that $PLL$ is very large when it is $\Gamma$ times more than the minimum cost of any edge in the current TRG, as given by the equation below:
\begin{equation}
PLL_{ij}^{(t)} = \begin{cases} 
      \mbox{FALSE} & \mbox{if}\,\, c_{ij}^{(t)} \leq \Gamma \min(\{c_{ik}^{(t)}:\forall k \in V \}) \\ 
      \mbox{TRUE} & otherwise \\ 
   \end{cases}
\label{eq:pll}
\end{equation}
where $\Gamma$ is a user defined constant that is based on system and environment factors such as battery remaining, terrain and navigation conditions. The sequence of values for $PLL_{ij}^{(1...t)}$ is recorded, and used to estimate the probability of the state variable $TNA_{ij}^{(t)}$, given by $p(TNA_{ij}^{(t)}|PLL_{ij}^{\{1:t\}})$, using the  Forward-Backward algorithm \cite{Russell09}.  The forward stage is given by the equation:

\begin{multline}
P(X^{(t+k+1)} \vert PLL_{ij}^{\{1:t\}}) = \sum_{TNA_{ij}^{(t+k)}} \bigg(P(X^{\{t+k+1\}} \vert TNA_{ij}^{(t+k)} ) \\
\quad \quad \quad \quad   P(TNA_{ij}^{(t+k)} \vert PLL_{ij}^{\{1:t\}}) \bigg) \nonumber
\end{multline}

and the backwards stage is given by:

\begin{equation}
\begin{split}
P(PLL_{ij}^{\{k+1:t\}} \vert X^{(k)}) = \sum_{TNA_{ij}^{(k+1)}} \bigg( P(PLL_{ij}^{(k+1)}) | TNA_{ij}^{(k+1)}) \\
\quad \quad \quad \quad P(PLL_{ij}^{\{k+2:t\}} \vert TNA_{ij}^{(k+1)})P(TNA_{ij}^{(k+1)} | X^{(k)})\bigg) \nonumber
\end{split}
\end{equation}

where $X^{(t)}$ is the combination of the set of state variables, $MO, SO$ and $TNA$ at time $t$, $PLL_{ij}^{\{1:t\}}$ is set of evidence (PLL observations) from time $1$ through $t$, and $TNA_{ij}^{(t)}$ is the value of the variable $TNA_{ij}$ at time $t$. Finally, to integrate the calculated value of $TNA_{ij}$ with the TRG edge $e_{ij}$, we model the task availability as probabilistic availability $p_{ij}$ of  TRG edge $e_{ij}$. $p_{i,j}$ gets a value $1$ when the robot is certain that the task is available and there exists a finite distance path to reach it, and, $0$ when the path to reach the task is infinitely large meaning the task is not available; intermediate probabilities represent the uncertainty of the task not being completed by other robots and still remaining available to the robot. $p_{ij}$ is calculated by normalizing Equation \ref{eq:pll}, given by: 

\begin{equation}
p_{ij} = \frac{p(TNA_{ij}|PLL^{\{1:t\}})}{\sum_j p(TNA_{ij}|PLL^{\{1:t\}})}.
\label{eq:hmm-pij}
\end{equation}
The normalization ensures that the robot has a probability $1$ of leaving TRG vertex $v_i$ through at least one of its incident edges.

\omitit{
\begin{figure}
\begin{center}
\includegraphics[scale=.8]{figs/hmm-tbn.eps}
\caption{{\small{Temporal element of Bayesian network}}}
\label{fig_hmm_temporal}
\end{center}
\end{figure}
}

\begin{figure}
  \centering
  \includegraphics[width=\linewidth]{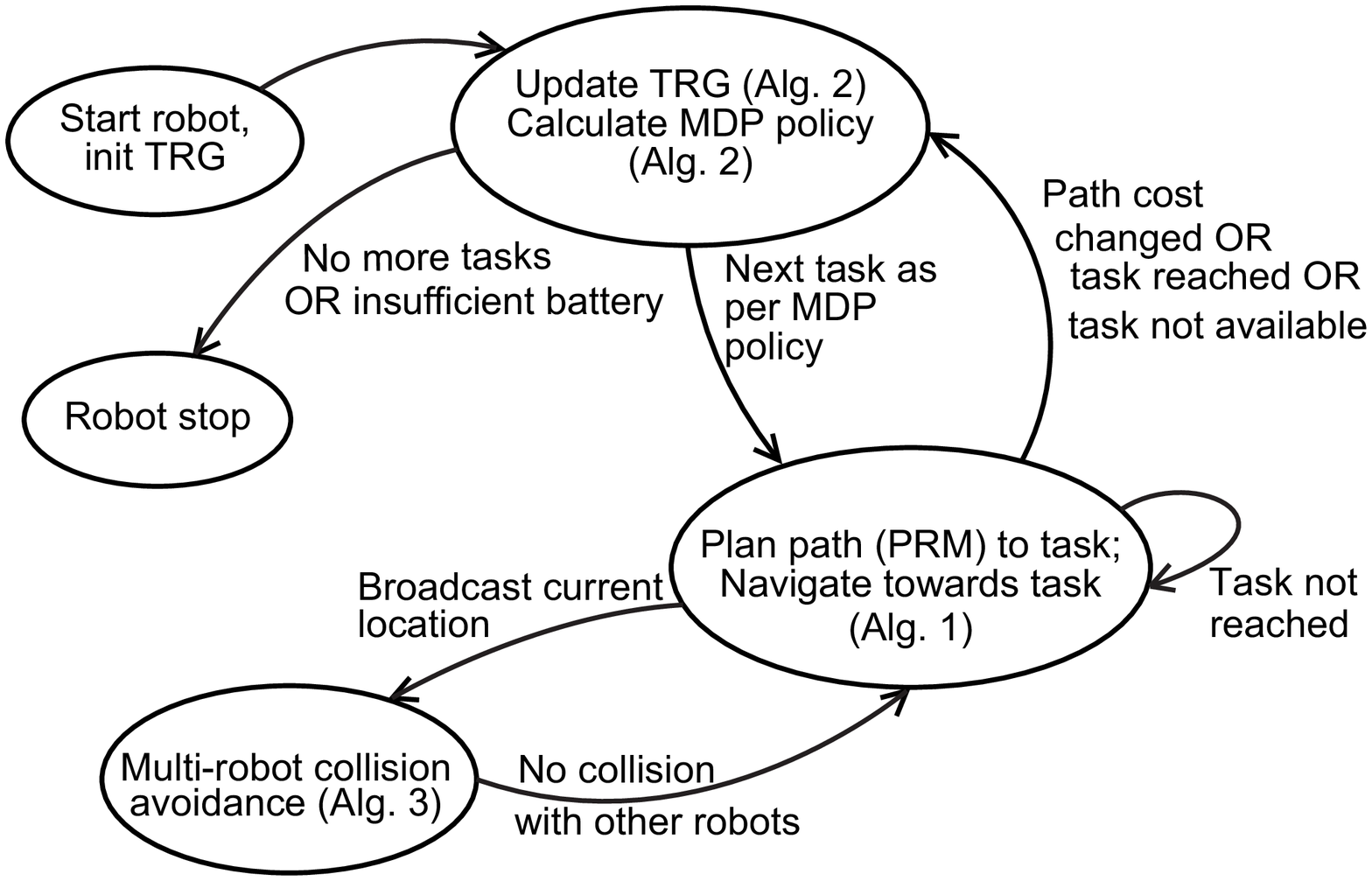}
  \caption{State diagram showing the operation of a robot using the different algorithms proposed.}
\end{figure}

\begin{algorithm}
{\textbf {TRGTaskSelect}}($TRG=<V, E, P, C>$) \\
\KwIn{$TRG$: task reachability graph}
Build initial PRM roadmap \\
Initialize MDP with current TRG information \\
Determine paths in robots configuration space using PRM planner
between all TRG edges $e_{ij}=(v_i,v_j) \in E$\\
$v_{curr} \leftarrow$ current position of robot\\
\While{$V$ is not empty AND battery available for next vertex}{
	$v' \leftarrow$ Next task as per MDP policy\\
	$path \leftarrow$ PRM path between $v_{curr}$ and $v'$\\
	\While{$v'$ not reached}{
		Update $v_{curr}$ using localization system \\ 
		Broadcast $v_{curr}$ to other robots\\
		//avoid collisions w. other robots (Alg.~\ref{algo_collision_avoidance})\\
		$collision \leftarrow$ coordinatePath($TRG, v'$) \\
                \If{$collision = FALSE$}{
                  \If{(taskCompleted message recd. from anothe robot) OR (new obstacle detected in robot's path)}{
		    $(v', path) \leftarrow$ updateTRG($TRG, v'$) (Alg.~\ref{algo_updateTRG})\\
                  }
		  Move along current segment of $path$
                }
	}
	Remove $v'$ from $V$ //reached $v'$ \\
    Communicate completion of task to all other robots 
}
\caption[Select Task]{{\small{Algorithm to select a task in the TRG using an MDP-based policy.}}}
\label{algo_task_select}
\end{algorithm}

\begin{algorithm}
{\textbf {updateTRG}}($TRG=<V, E, P, C>, v'$) \\
\KwIn{$TRG$: task reachability graph; $v'$: destination TRG vertex}
\KwOut{$v$: destination TRG vertex, $path$: path to destination TRG vertex}		
	update $V$ removing completed tasks, if any\\
        \For{ $(v_i, v_j) \in E$ }{
          $path' \leftarrow$ replan path from $v_i$ to $v_j$ (PRM)\\
          $c_{ij} \leftarrow $pathLength($path'$) \\
        }
	$path \leftarrow$ replan path from $v_{curr}$ to $v'$ using PRM-planner\\
        Generate observations $PLL_{ij}$ for every $e_{ij}$ in TRG\\
        Update $p_{ij}$ using HMM in Eqn. \ref{eq:hmm-pij} for every $e_{ij}$ \\
	Update MDP, TRG with new values of $V$, $p_{ij}$, $c_{ij}\,$ for every $e_{ij}$\\
	$v_{new} \leftarrow$ Next task as per updated MDP policy (Egn. \ref{eq:mdp})\\ 
	\If {$v_{new} = \{\emptyset\}$}{
		return null; // No more tasks 
	}
	\If { $v_{new} \not = v'$ } {
		$v' \leftarrow v_{new}$; // Switch tasks\\
		$path \leftarrow$ PRM path between $v_{curr}$ and $v'$
	}
return $v', path$
	
\caption[Update TRG]{{\small{Algorithm to update TRG and path when TRG vertices
are removed (task completed) or a new obstacle is detected that
triggers a path re-calculation.}}}
\label{algo_updateTRG}
\end{algorithm}

\subsection{TOP-U Solution using Markov Decision Process}

Following the update of the edge costs and availabilities, the robot has to select the TRG edge with the minimum expected cost, weighted with availability to solve the TOP-U problem (Equation \ref{eq:topu}). Because of the uncertainties in edge costs and availabilities, a Markov Decision Process(MDP) is used to do this. An MDP~\cite{Russell09} consists of a set of states, a set of actions to transition between states, along with a probability distribution and reward for each action at each state. The output of an MDP is a policy that prescribes an action at each state, which maximizes the cumulative, expected reward to the robot to reach a desired or goal state from its current state. A more thorough discussion on MDPs and solution techniques is given in ~\cite{Russell09}. For our TOP-U problem, the TRG's vertices, $V$, represent the MDP's states, the set of actions at each state (TRG vertex) of the MDP correspond to the edges from that TRG vertex, TRG edge availabilities give the transition probabilities between MDP states, while the inverse of the TRG edge costs correspond to the reward for reaching each state in the MDP (lower edge costs corresponds to larger rewards). The policy calculated by the MDP gives the maximum expected reward (or minimum expected cost, weighted by availability) for the robot to visit the TRG vertices. The MDP is solved using the value iteration algorithm, that solves the following equation: 

\begin{equation}
U(v_i) = c_{curr,i}^{-1} + \gamma \max_{e_{ij} \in E} \sum_{v_k} P(v_k \vert v_i, e_{ij}) U(v_k)
\label{eq:mdp}
\end{equation}

where $c^{-1}_{curr,i}$ is the inverse expected cost from the robots current location to $v_i$, $\gamma$ is a user-defined, reward-discount factor and $P(v_k \vert v_i, e_{ij})$ is the probability that the robot will reach task $v_k$ when starting at task $v_i$ and attempting to follow the edge $e_{ij}$ towards task $v_j$ which may or may not be the same task as $v_k$. $v_k \not = v_j$ happens if the robot was to attempt going to task $v_j$, but due to obstacles, communications, etc. it determines that it is better to instead go to task $v_k$.  The equation for $P(v_k \vert v_i, e_{ij})$ is given below, which if $v_k = v_j$, the probability is the edge availability, and if $v_k \not = v_j$, then it is the probability of the edge not being available distributed evenly to the remaining tasks.

\begin{equation}
P(v_k \vert v_i, e_{ij}) = \begin{cases} 
1-p_{ij} & \text{if $j=k$} \\
\frac{p_{ij}}{\vert V \vert - 1} & \text{otherwise}
\end{cases}
\label{eq:mdp-prob}
\end{equation}

\subsection{Robot Navigation and Multi-robot Path Coordination Algorithms}

The main algorithm used by a robot for selecting tasks to visit is shown in Algorithm \ref{algo_task_select}. The main idea of the algorithm is to select the task, $v'$, determined by the MDP policy, and plan a path to reach it. If the path results in potential collisions with other robots' paths, path conflicts are resolved (line $11$). Every time the path cost to a task changes due to obstacles, or a task completed message from another robot is received, a TRG update is triggered (line $14$). This might result in switching the task the robot is headed to. The robot continues to move towards its currently selected task until it is reached and upon reaching the task its removes its vertex from the TRG and broadcasts task complete message to other robots (lines $16-20$).

The algorithm used to update the TRG is shown in Algorithm \ref{algo_updateTRG}. When a robot's TRG vertex set or path costs on the TRG change, it calculates a new navigation path to its destination vertex $v'$ (lines $2-7$) and new edge availability values using its current perception in the HMM (line $8-9$). These updated values are incorporated into the MDP and the MDP's policy is recalculated to yield the new destination vertex (line $10-11$). If the recalculated policy prescribes a new target vertex, $v_{new}$then the robot performs a task switch and its destination vertex is changed from $v'$ to $v_{new}$ (line $15-17$). The algorithm also handles the case where all tasks in a robot's TRG might get completed by other robots before it reached those tasks; in that case the algorithm returns a null vertex and empty path(lines $12-13$) so that the robot stops.

\begin{algorithm}[htb!]
  \caption[Collision Avoidance 3]{{\small{Modified Algorithm to avoid collisions between robots in close proximity of each other.}}}
  \label{algo_collision_avoidance}
        {\textbf {coordinatePath}} \\
        \KwIn{$v'$: destination TRG vertex; $TRG$: task reachability graph}
        \KwOut{Is the robot currently in a collision}
        \If {another robot within $O_i$}{
          stop \\
          build/update collision shape \\
          \If {previous winner token released}{
            prio $\leftarrow$ robot id \label{line:StoW} \\
          }
	  //priority is either robot id or $\infty$\\ 
	  send/receive priority to/from other robots within $O_i$\\
          select robot with lowest prio in $O_i$ as winner \label{line:WtoL}  \\
          \If{I was winner, but lost this round}{
            transfer winner token to winning robot \label{line:LtoW}  \\
          }
	  \If {I am winner}{
            $(v', path) \leftarrow$ updateTRG($TRG,v'$)\\
	    //other robots considered as static obstacles in PRM\\
	    \If{$v' = null$}{ 
              //No more valid paths available to the robot\\
	      prio $\leftarrow \infty$ \label{line:LtoS} \\
              }
	    \Else {
	      Move along current segment of $path$\\
	    }
            \If{outside $O_i$}{
	      release winner token \label{line:LtoX} \\	
            }
          }
	  \Else {
            \If{All priorities in collision shape are $\infty$}{ \label{line:jointPlanning} 
              $path \leftarrow$ performJointPlanning(TRG)\\
              \If{$path = null$}{
                exit FAILURE
              }
              Set all robots prio to robot id \label{line:jointPathFound}\\
              return TRUE
            }
	  }
          return FALSE\\
        }
        return FALSE\\
\end{algorithm}

\subsection{Coordinating paths between robots to avoid collisions} If robots determine their paths individually using Algorithm~\ref{algo_task_select}, it could lead to robot collisions when the planned paths of two or more robots intersect with each other. To avoid this scenario, we have used a collision avoidance algorithm shown in Algorithm~\ref{algo_collision_avoidance}. Each robot uses the locations broadcast by other robots to check if there are other robots within a radius of $r_{coll}$, called the collision circle, of itself (lines $2$). When a set of robots are within the collision circle of each other, all the robots stop and the robots exchange their identifiers, representing their priorities, with each other. A leader election algorithm called the bully algorithm~\cite{bully82} is then used to select the robot with the highest priority as the winner. The winner robot holds the winner token, which gives it the right to move (lines $3-7$). All other robots in the collision circle, which do not hold the winner token, remain stationary (line $24-26$). The winner robot uses the PRM planner in conjunction with updating the TRG using Algorithm \ref{algo_updateTRG} to find a path to its destination vertex $v'$. The path returned by the PRM planner is executed and the moving robot releases the winner token once it is outside its collision circle (lines $9, 14-22$). If the PRM planner is not able to find a path to the goal, e.g., if the goal is unreachable because there is another robot within the collision circle that is stopped right at the goal location, the moving robot relinquishes its right to move by setting its priority to a high value ($\infty$) (lines $11-12$). Another robot from within the set of stopped robots gets a chance to run the bully algorithm and attempts to move. This protocol ensures that at least one robot exits the collision circle with each execution of the bully algorithm, and finally there is only one robot left inside the collision circle. This robot then reverts to using Algorithm \ref{algo_task_select} to plan its path.

\section{Theoretical Results}
\label{sec_proofs}
\begin{figure}
  \begin{tabular}{cc}
    \includegraphics[width=1.2in]{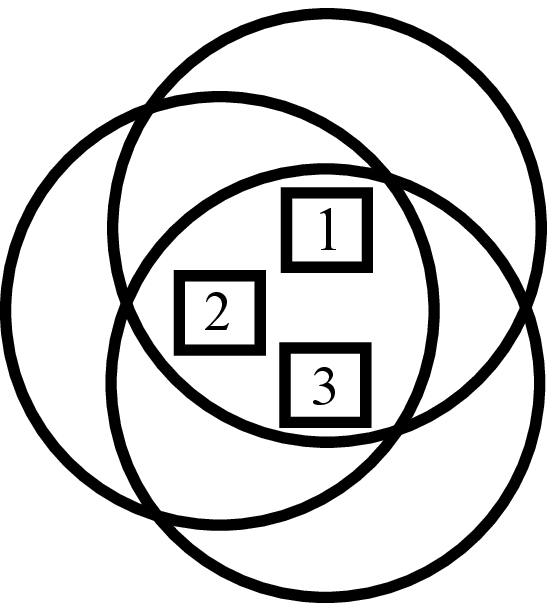} &
    \includegraphics[width=1.0in]{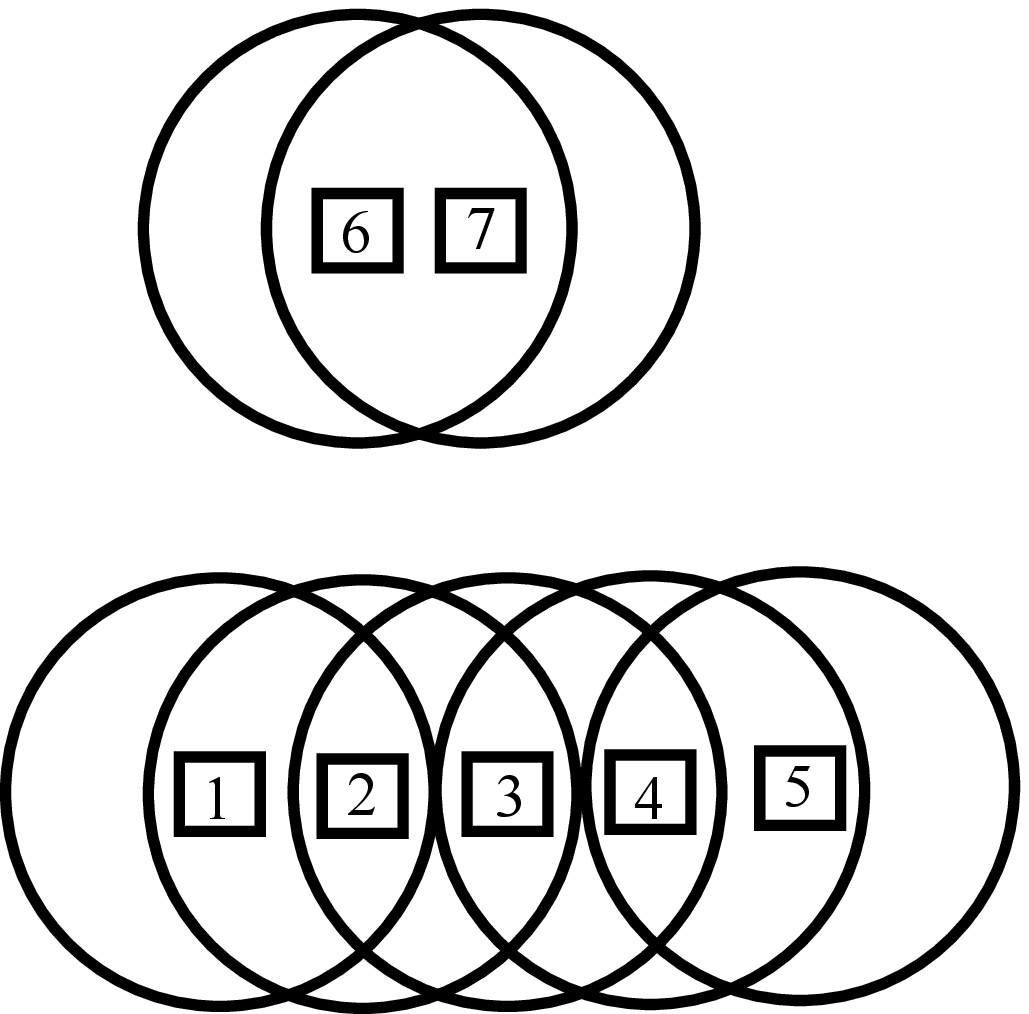}\\
    {\small{(a)}} & {\small{(b)}}
  \end{tabular}
  \caption{{\small{(a) Example collision circle.  The boxes represent
        the robots, the numbers are the robot's ids, and the circles
        represnt the circle of size $r_{coll}$ around the robot.
        Collision circle for robot 1 is $O_1 = \{1, 2, 3\}$, (b) Two
        example collision shapes.  The boxes represent the robots, the
        numbers are the robot's ids, and the circles represnt the
        circle of size $r_{coll}$ around the robot.  Collision shape
        for robot 1 is $C_1 = \{1, 2, 3, 4, 5\}$, Likewise collision
        shape for robot 6 is $C_6 = \{6,7\}$.}}}
\label{fig:collisionCircleAndShape}
\end{figure}

In this section, we prove some properties related to our proposed algorithms.

\subsection{Task Selection Algorithm Properties}
\begin{theorem}
Algorithm {\em TRGTaskSelect} finds a solution that is admissible, that is, it never overestimates the expected cost to a task calculated using the TRG.
\label{theo:admissible}
\end{theorem}

\begin{proof} (By Contradiction.) Let us suppose there is another algorithm $A'$ that selects TRG edges with lower expected cost than Algorithm {\em TRGTaskSelect}. Let $v_j$ and $v_{j'}$ ($v_j \neq v_{j'}$) respectively denote the task (vertex) selected by Algorithm {\em TRGTaskSelect} and $A'$, when the robot is at vertex $v_i$. Using Equations \ref{eq:mdp} and \ref{eq:mdp-prob}, the MDP in Algorithm {\em TRGTaskSelect} (Alg. \ref{algo_updateTRG}, line $11$) calculates the expected cost to reach $v_j$ from $v_i$ using TRG edge $e_{ij}$ as $(1-p_{ij})c_{ij}$. Algorithm $A'$, which does not follow the vertex recommendation made by Algorithm {\em TRGTaskSelect}, selects a different TRG edge $e_{ij'}$, which has cost $c_{ij'}$ with probability $(1-p_{ij})$ and TRG edge $e_{ij}$ with probability $p_{ij}$. The robot's expected cost to reach $v_{j'}$ using Algorithm $A'$ is then $(1- p_{ij})c_{ij'} + p_{ij}c_{ij}$. Since, by Equation \ref{eq:mdp}, the MDP selects the edge with the minimum cost, it follows that, $c_{ij}< c_{ij'}$; let $c_{ij'} = c_{ij} + \Delta$, $\Delta > 0$. By our assumption that the expected cost calculated by $A'$ is lower than that of Algorithm {\em TRGTaskSelect}, we get:
\begin{eqnarray}
(1- p_{ij})(c_{ij} + \Delta) + p_{ij}c_{ij} < (1-p_{ij})c_{ij} \\
\mbox{or,} \quad (1- p_{ij}) \Delta + p_{ij}c_{ij} < 0 \nonumber \\
\mbox{or,} \quad p_{ij} > \frac{\Delta}{\Delta -c_{ij}} \nonumber \\
\mbox{or,} \quad p_{ij} > \frac{1}{1 -\frac{c_{ij}}{\Delta}} \nonumber
\label{eq:trgproof1}
\end{eqnarray}
Because, by definition, $c_{ij}$ and $\Delta$ are both $>0$, we get $p_{ij} > 1$, which is not valid as $p_{ij} \in [0,1]$. Therefore, the expected cost returned by $A'$ cannot be lower. Hence proved.
\end{proof}

\begin{lemma}
\label{lemma:admissible}
The solution found by Algorithm {\em TRGTaskSelect} remains admissible when there are errors in the edge cost and availability estimates.
\end{lemma} 

\begin{proof}
Let $c_{ij, est}, \epsilon_c$ and $p_{ij,est}, \epsilon_p$ denote the estimated values and errors in edge costs and availabilities returned by the PRM and HMM respectively,(using Algorithm \ref{algo_updateTRG}, lines $7$ and $9$), where $\epsilon_c \in \mathbb{R}, \epsilon_c \ll c_{ij}$ and $p_{ij,est} \in [0,1]$ and $0 \leq p_{ij} + \epsilon_p \leq 1$. If the true cost and availability of edge $e_{ij}$ is denoted by $c_{ij}$ and $p_{ij}$ respectively, $c_{ij,est} = c_{ij} + \epsilon_c$ and $p_{ij,est} = p_{ij}+\epsilon_p$. For Algorithm {\em TRGTaskSelect} to not overestimate edge costs with these estimated values Theorem \ref{theo:admissible} should hold when they are used in Equation \ref{eq:trgproof1}. That is, $(1- p_{ij,est})(c_{ij,est} + \Delta) + p_{ij,est}c_{ij,est} > (1-p_{ij,est})c_{ij,est}$. Substituting above values of $c_{ij,est}$ and $p_{ij,est}$ we get:

\begin{equation}
(p_{ij} + \epsilon_p)(c_{ij}+\epsilon_c) + (1-(p_{ij}+ \epsilon_p))\Delta > 0
\label{eq:trgproof2}
\end{equation} 

As $p_{ij}$ has limited domain of $[0,1]$, and the above equation is linear in $p_{ij}$, we can check the two boundary points of the domain; if the inequality is satisfied for both boundaries, then it is satisfied for all points inside the domain. Note that then when $p_{ij}=0$, $\epsilon_p \in [0,1]$ and when $p_{ij}=1$, $\epsilon_p \in [-1,0]$. Substituting $p_{ij}=0$ in Equation \ref{eq:trgproof2}, we get $\epsilon_p(c_{ij}+\epsilon_c + (1-\epsilon_p)\Delta > 0$. Since $\epsilon_p \in [0,1]$ and $c_{ij},\Delta >0$, each of the terms on l.h.s. of the last inequality is $>0$ and the inequality holds. Similarly, substituting $p_{ij}=1$ in Equation \ref{eq:trgproof2}, we get $(\epsilon_p + 1) (c_{ij}+\epsilon_c) - \epsilon_p \Delta >0$.  Since $\epsilon_p \in [-1,0]$ when $p_{ij}=0$, we substitute the boundary values of $\epsilon_p$ in the last equation; when $\epsilon_p = -1$, we get $\Delta >0$ which is valid from the definition of $\Delta$; when $\epsilon_p=0$, we get $(c_{ij}+\epsilon_c)>0$, which is valid, because, by definition $\epsilon \ll c_{ij}$.
\end{proof}

\begin{theorem}
\label{theo:consistent}
The solution found by Algorithm {\em TRGTaskSelect} is consistent.
\end{theorem}

\begin{proof}
Suppose the robot is at vertex $v_i$ and $v_j$ is the next vertex selected by Algorithm {\em TRGTaskSelect}. 

For consistency property, we need to show that $c_{ij} < c_{ik} + c_{kj}$ for any $v_k \neq v_i, v_j$. We prove by contradiction - suppose $c_{ij} > c_{ik} + c_{kj}$. Note that the TRG is a complete graph and vertices $v_i, v_j$ and $v_k$ form a triangle. Consequently, $||e_{ij}|| < ||e_{ik}||+ ||e_{kj}||$, where $||e_{ij}||$ is the Euclidean distance between $v_i$ and $v_j$. From Section \ref{sec:edgecost}, the path $\rho_{ij}$ for navigating the robot along TRG edge $e_{ij}$ has cost $c_{ij}$; it is found by the path planner and from Equation \ref{eq:prm-cost}, it is guaranteed to be the minimum cost, collision-free path connecting $v_i$ and $v_j$. In other words $\nexists q \in Q_{free}$, satisfying $c_{ij} > c_{iq} +c_{qj}$. Because a task location has to be in the free space in the environment, $v_k \in Q_{free}$ and, so, the last inequality is valid for $q=v_k$. Therefore, $c_{ij} \ngtr c_{ik} +c_{kj}$. Therefore, our assumption was invalid. Hence proved.
\end{proof}

\begin{theorem}
The solution found by Algorithm {\em TRGTaskSelect} is optimal.
\end{theorem}
\begin{proof}
From Theorems \ref{theo:admissible} and \ref{theo:consistent} it follows that the solution found by Algorithm {\em TRGTaskSelect} is both admissible and consistent. Therefore, the solution is optimal.
\end{proof}

Similar to Lemma \ref{lemma:admissible}, it can be shown that the solution found by Algorithm {\em TRGTaskSelect} is consistent and hence optimal, even with errors in path cost estimates.

\subsection{Completeness of Coordination Algorithm}

Next, we analyse the synchronization properties of our proposed multi-robot coordination algorithm to show that it does not give rise to deadlock or livelock conditions between robots, resulting in their inability to move and reach tasks.  To facilitate this analysis, we consider the movement of the robots between sets, corresponding to robot states defined by the algorithm.

Let dist($a_i$, $a_j$) denote the Euclidean distance between robots $a_i, a_j \in R$.  We define the collision circle for robot $a_i$ as $O_i = \{ a_j : a_j \in R,\,$dist($a_i$, $a_j$)$ \leq r_{coll} \}$, where $r_{coll}$ is the distance away from robot $a_i$ that we consider robots to be in immediate risk of collision.  An example of a collision circle can be seen in Figure~\ref{fig:collisionCircleAndShape}(a), where robot 1 has robots 2 and 3 in its collision circle.  This means that $O_1 = \{1, 2, 3\}$. We next define the concept of a collision shape.  A collision shape is the group of all robots that are either in each others collision circle, or, through sharing collision circles with other robots, can reach the collision circle of another robot without exiting any overlapping collision circles.  As shown in Figure~\ref{fig:collisionCircleAndShape}(b), robot 3 is not in the collision circle of robot 1, however it is in the same collision shape as robot 1, because by traveling through the collision circle of robot 2, robot 1 can reach the collision circle of robot 3.  On the other hand, robot 6 is not in the collision shape of robot 1 because there is no way to move into the collision circle of robot 6 without leaving collision circles.  To help define the collision shape, we first define a recursive helper set $H_i^{[n]}$, corresponding to robot $a_i \in R$ at recursive step $n$, as $H_i^{[n+1]} = H_i^{[n]} \bigcup_{j \in   H_i^{[n]}} O_j$ and $H_i^{[0]} = O_i$.  As $n \rightarrow \infty$, $H_i^{[n]}$ becomes the set of all robots in the current collision shape. 
We can now define a collision shape as:
\begin{definition}
Let $H_i^{[n]}$ denote a helper set of robot $a_i$, as defined above. Then collision shape $C_i = \lim_{n \rightarrow \infty} H_i^{[n]}$ 
\end{definition}

The collision circle and collision shape of a robot get updated as the algorithm proceeds.  We use $O_i^{(t)}$ and $C_i^{(t)}$ to denote the collision circle and collision shape of robot $i$ at round $t$ respectively.
\omitit{
Robots can enter or exit $C_i$ or
$O_i$, in successive rounds.  Formally, if we let $U_i \subseteq C_i^{(t)}$
denote the set of robots from $C_i^{(t)}$ that have exited the
collision circle, and let $U_i' \not \subseteq C_i^{(t)}$ denote the
set of robots not in $C_i^{(t)}$ that are entering the collision
circle, we can define the collision circle at the next iteration as
$C_i^{(t+1)} = (C_i^{(t)} \backslash U_i') \union U_i$.}  We now define
three subsets of $C_i^{(t)}$ used in our analysis.

\begin{itemize}
\item $W_i^{(t)} \subseteq C_i^{(t)}$: The set of all robots in the collision shape of robot $a_i$ waiting for their turn to move
\item $L_i^{(t)} \subseteq C_i^{(t)}$: The set of all robots in the collision shape of robot $a_i$ that are allowed to move
\item $S_i^{(t)} \subseteq C_i^{(t)}$: The set of all robots in the collision shape of robot $a_i$ that have surrendered their right to move 
\omitit{\item $X_i^{(t)} \subseteq C_i^{(t)}$: The set of all robots that have left the collision shape of robot $a_i$}
\end{itemize}

We also define a set $X_i^{(t)}$ as the set of all robots that were
originially in the collision shape of robot $a_i$, but have since
left.  Formally, $X_i^{(t)} = \bigcup_{j=0}^{t} C_i^{(j)} \backslash
C_i^{(t)}$.  The elements of the above sets are mutually exclusive and
the subsets fully partition the set $C_i^{(t)}$, in other words the
following properties hold: 
$W_i^{(t)} \cap L_i^{(t)} = \{\emptyset\}$, 
$W_i^{(t)} \cap S_i^{(t)} = \{\emptyset\}$,  
$L_i^{(t)} \cap S_i^{(t)} = \{\emptyset\}$,  
$W_i^{(t)} \cup L_i^{(t)} \cup S_i^{(t)} = C_i^{(t)}$, and
$X_i^{(t)} \cap C_i^{(t)} = \{\emptyset\}$

\begin{lemma}
If during round $t$, there are robots currently waiting to be allowed to move, then in the next round, there has to be at least one robot that is selected for movement. 
That is, if $W_i^{(t)} \not = \{\emptyset\}$ then $L_i^{(t+1)} \not = \{\emptyset\}$.
\label{lemma:emptyL}
\end{lemma}

\begin{proof}
Line~\ref{line:WtoL} selects one element $a_j$ from each $O_i^{(t)}$ where $a_j \in C_i$ for membership in $L_i^{(t)}$.  The robot selected can be a member of one of two sets, $a_j \in L_i^{(t)}$ or $a_j \in W_i^{(t)}$.  If $a_j \in L_i^{(t)}$, then $L_i^{(t+1)} \not = \{\emptyset\}$, because $a_j$ will remain in $L_i^{(t)}$. And if $a_j \in W_i^{(t)}$, then because $a_j \not = \emptyset$ in the next time step, $L_i^{(t+1)} = L_i^{(t)} \union \{a_j\}$, which means $L_i^{(t+1)} \not = \{\emptyset\}$.  Therefore, we can conclude that if $W_i^{(t)} \not = \{\emptyset\}$, then in the next time step $L_i^{(t+1)} \not = \{\emptyset\}$.
\end{proof}

\begin{lemma}
{\em{coordinatePath}} deadlocks when no robots are allowed to move and there are no robots available to be selected for movement. This occurs when $L_i^{(t)} = W_i^{(t)} = W_i^{(t-1)} = \{\emptyset\}$, and $C_i^{(t)} \not = \{\emptyset\}$.
\label{lemma:deadlock}
\end{lemma}

\begin{proof}
As, $C_i^{(t)} = \{\emptyset\}$ is the termination case, $C_i^{(t)} \not = \{\emptyset\}$ is a necessary condition. When the robots deadlock, that means that for the rest of time, no robots are allowed to move.  Or, mathematically, $\exists t'$ such that $L_i^{(t)} = \{\emptyset\}$ $\forall t \geq t'$. By Lemma~\ref{lemma:emptyL}, this implies that $W_i^{(t)} = \{\emptyset\}$ $\forall t \geq t'-1$.  This implies that for each round $t$ that is in deadlock $L_i^{(t)} = W_i^{(t)} = W_i^{(t-1)} = \{\emptyset\}$. 
\end{proof}

\begin{theorem}
  Algorithm coordinatePath does not deadlock.
\end{theorem}

We analyse the operation of the collision avoidance algorithm as a method to move robots between two sets, $C_i$ and $X_i$. Initially, all robots in $C_i$ are placed in $W_i$, and from $W_i$, movements through $S_i$ and $L_i$ are possible until the movement into $X_i$ is possible. Based on the above descriptions, a deadlock can only occur when no robot is allowed to move, meaning that no winners have been selected, and there are no available waiting robots to become winners.

\omitit{Robots  start in $W_i$ and by moving through other sets, eventually are placed in $X_i$.} 

\begin{proof}
By Lemma~\ref{lemma:deadlock}, {\em{coordinatePath}} deadlocks when $L_i^{(t)} = W_i^{(t)} = W_i^{(t-1)} = \{\emptyset\}$ and $C_i^{(t)} \not = \{\emptyset\}$.  As $C_i^{(t)} \not = \{\emptyset\}$, and $C_i^{(t)} = L_i^{(t)} \union W_i^{(t)} \union S_i^{(t)}$, this implies that $C_i^{(t)} = S_i^{(t)}$. Line~\ref{line:jointPlanning} tests for this condition when all robots remaining in the collision shape have a priority of $\infty$. This causes joint configuration space planning to be called, which is guaranteed to find   paths for all robots, if such paths exist, or a failure when paths do not exist, but does not deadlock.  If paths do exist, all of the robots in the suspended state $S_i^{(t)}$ are transitioned into the waiting state $W_i^{(t+1)}$, on line~\ref{line:jointPathFound}. In other words $S_i^{(t+1)} = \{\emptyset\}$ and $W_i^{(t+1)} = S_i^{(t)}$.  The conditions for deadlock given in Lemma~\ref{lemma:deadlock} no longer holds as $W_i^{(t+1)} \not = \{\emptyset\}$.
\end{proof}

\begin{figure}
  \centering
  \includegraphics[width=0.25\textwidth]{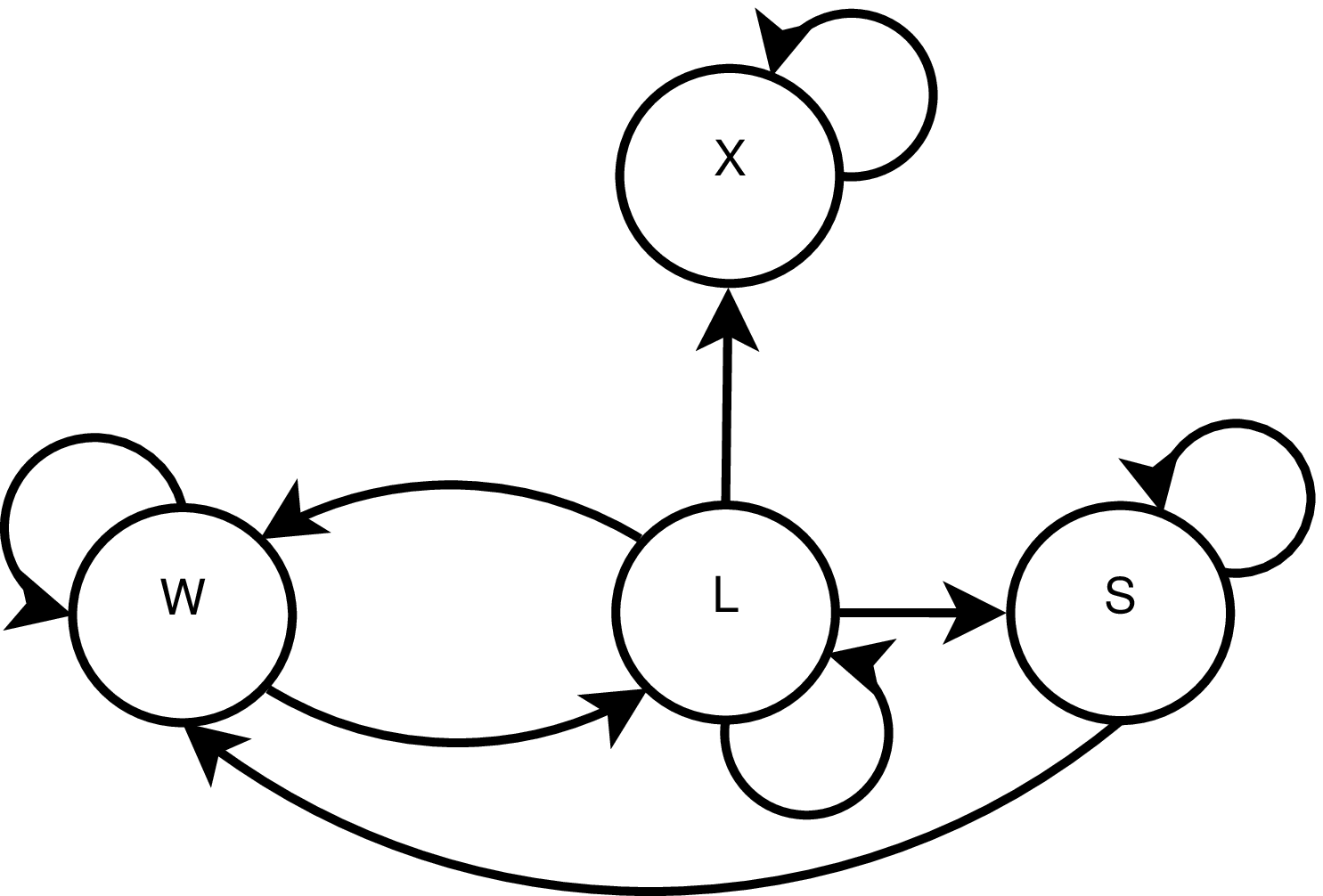}
  \caption{State diagram for Algorithm {\em{coordinatePath}}.}
  \label{fig:coordPathStateDiag}
\end{figure}

\begin{definition}
  A livelock occurs for algorithm {\em{coordinatePath}} when there is
  no robot in the collision shape that is able to reach the sink state
  X through a finite set of state transitions across the W, L, and S
  states.
\end{definition}

\begin{lemma}
\label{lemma:livelockCycleS}
Any cycle in the state transition graph that contains state $S$ cannot be 
a livelocking cycle
\end{lemma}

\begin{proof}
Once a robot enters state $S$, there are only two scenarios, stay in $S$ or exit state $S$ by going only to state $W$.  In the first case, there will come a time where either all robots have entered $S$. This triggers a joint planning, which is guaranteed to not result in livelock.  In the second case, the only way a robot in $S$ is allowed to reach $W$ is if another robot makes a transition from $L$ to $X$. This implies, from the definition of livelock, that algorithm {\em{coordinatePath}} is not in livelock, because robots outside of the coordinate path algorithm can make progress towards their tasks uninhibited.
\end{proof}

\begin{lemma}
\label{lemma:finiteCycles}
There are only a finite number of cycles or paths through the state graph in Figure~\ref{fig:coordPathStateDiag} that could feasibly result in a livelock scenario
\end{lemma}

\begin{proof}
Based on Lemma~\ref{lemma:livelockCycleS}, any cycle that moves through $S$ can be eliminated from the list, which leaves only two possible states to move through. Then based on the properties of finite state machines, all cycles which result in multiple visits of the same state in a row can be collapsed into the equivalent cycle visiting that state only once.  As such, this leaves us with only three potential cycles, $W$, $L$, and $W \rightarrow L$.
\end{proof}

\begin{lemma}
\label{lemma:stayInL}
It is not possible for robots to eternally stay in state $L$.
\end{lemma}

\begin{proof}
All robots in $L$ must remain moving and may not remain stopped for any time not required for planning or sensing. The paths returned by PRM are the shortest path in the roadmap to the goal possition.  Based on these two facts, the robot must either move out of the collision shape, or enter a collision with another robot causing its transition either into $W$ or $S$ depending on if the collision was with a higher priority robot, or if the collision caused all feasible paths for the robot to become blocked.
\end{proof}

\begin{lemma}
\label{lemma:cycleWL}
It is not possible for robots to eternally cycle between sets $W$ to $L$ and back to $W$.
\end{lemma}

\begin{proof}
Transitions from $W$ to $L$ occur when a robot is the highest priority robot in its collision circle $O$. Transitions from $L$ back to $W$ occur when a robot loses its position as the highest priority robot in its collision circle, when it was moving. The robot would still be able to move towards its goal if the higher priority robot was not there. Also, consider the worst case scenario, where, as each robot selected from $W$ to get into $L$, moves and encounters a robot of higher priorty, causing it to transition from $L$ back to $W$.  Even assuming that $L$ contains only one robot at a time, as the set of robots is finite, and all robot ids are unique inside $W$ or $L$, there must exist at least one robot with highest priority, which, once entering $L$, can not be moved back to $W$ as there is no other higher prority robot to prevent its movement. This highest priority robot only has two destination states that it can reach, $S$ or $X$; either of which breaks the loop $W-L-W$.  Therefore, a scenario where a robot cycles between $W-L-W$ states cannot exist.
\end{proof}

\begin{lemma}
\label{lemma:stayInW}
It is not possible for robots to eternally stay in the state $W$
\end{lemma}

\begin{proof}
This would be considered a deadlock state.  For robots to stay in $W$, there must be a higher priority robot inside their collision circle that is selected to move to $L$.  All members of $L$ must remain in motion, otherwise they become members of $S$, and no longer the highest priority robot in the collision circle, causing a new robot
from $W$ to be selected.  Robots in $L$ can not remain perennially inside the collision circle as the path planner PRM generates a path to goal. If a feasible path is found by the planner, the robot will leave the collision circle by transitioning from $L$ to $X$.  On the other hand, if a feasible path cannot be determined, the robot enters
$S$.
\end{proof}

\begin{theorem}
Algorithm {\em{coordinatePath}} does not result in livelock.
\end{theorem}

\begin{proof}
Livelock is when the robots are still able to move, however they make no progress towards completing their goals.  To do this, the robots must remain in states that allow motion, but never reach their goal locations (tasks), or exit their collision shapes.  Figure~\ref{fig:coordPathStateDiag} shows all of the possible transitions between states of the {\em{coordinatePath}} algorithm.  Through inspection of this figure, we can determine that livelock occurs when the robots follow any cycle in the graph without any robots entering state $X$. This provides the following livelock possibilities for the robots:
\begin{enumerate}
\item \label{item:W} Stay in $W$
\item \label{item:L} Stay in $L$
\item \label{item:WL} Cycle between states states $W$ and $L$
\end{enumerate}

Lemma~\ref{lemma:stayInL} shows that the option \ref{item:L} cannot exist.  Lemma~\ref{lemma:cycleWL} shows that option \ref{item:WL} cannot exist. And lemma~\ref{lemma:stayInW} shows that option \ref{item:W} cannot exist. Based on this, we can conclude that {\em{coordinatePath}} does not livelock.
\omitit{And lemma~\ref{lemma:cycleWLS} shows that option \ref{item:WLS} cannot exist.} 
\end{proof}

\section{Experimental Results}
\label{sec_expts}

\begin{table}
\centering
\caption{Mapping between environment parameters and Load values}
\label{tab:loadMapping}
\begin{tabular}{|l|l|l|l|}
\hline
\textbf{Tasks} & \textbf{Robots} & \textbf{Visits} & \textbf{Load} \\ \hline
5  & 3 & 1 & 1.67  \\ \hline
10 & 3 & 1 & 3.33  \\ \hline
5  & 1 & 1 & 5     \\ \hline
10 & 3 & 2 & 6.67  \\ \hline
10 & 1 & 1 & 10    \\ \hline
15 & 3 & 2 & 10    \\ \hline
15 & 1 & 1 & 15    \\ \hline
15 & 3 & 3 & 15    \\ \hline 
\end{tabular}
\end{table}

\begin{figure}
\centering
\begin{tabular}{c c}
\includegraphics[width=0.45\linewidth]{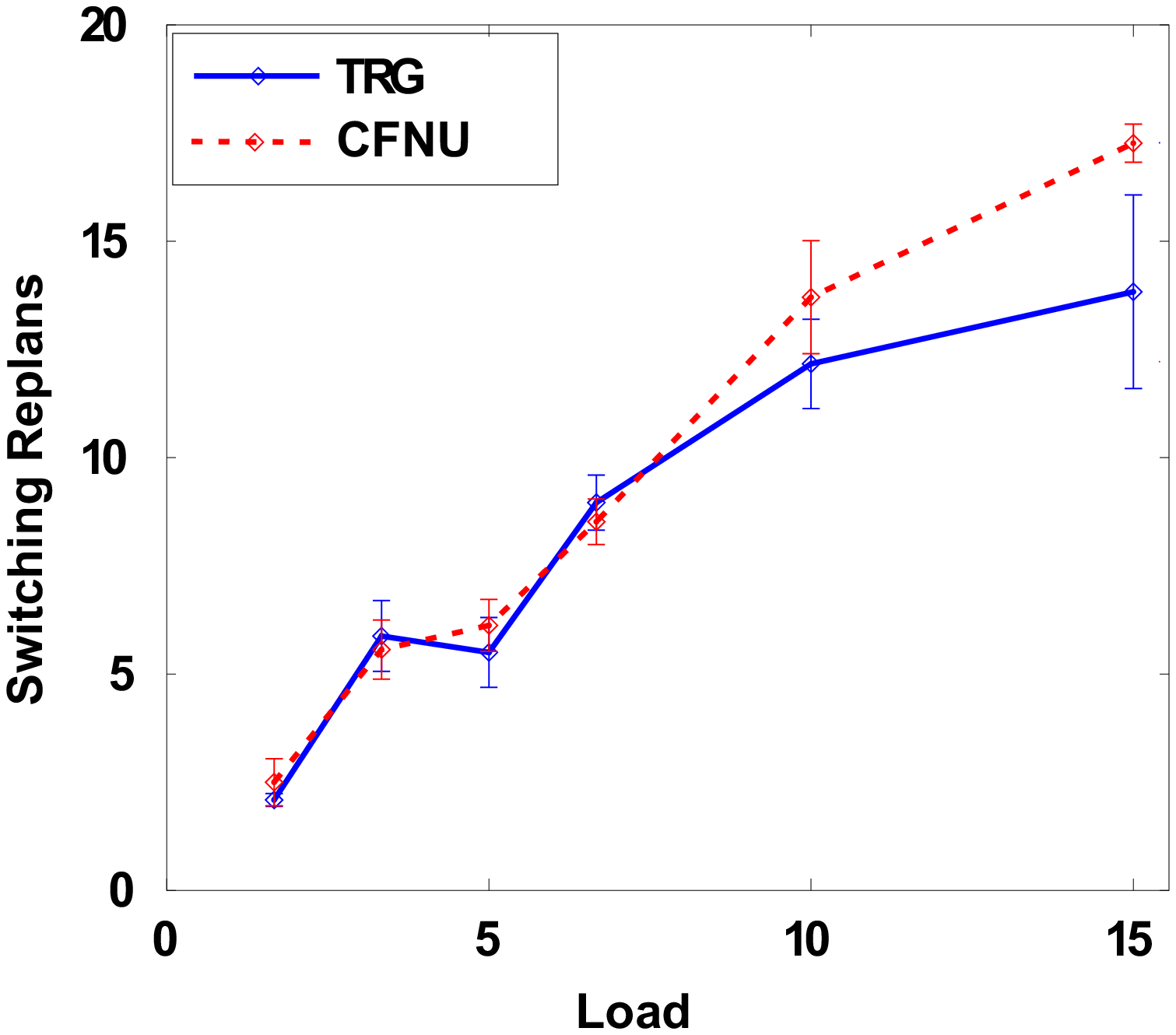} &
\includegraphics[width=0.45\linewidth]{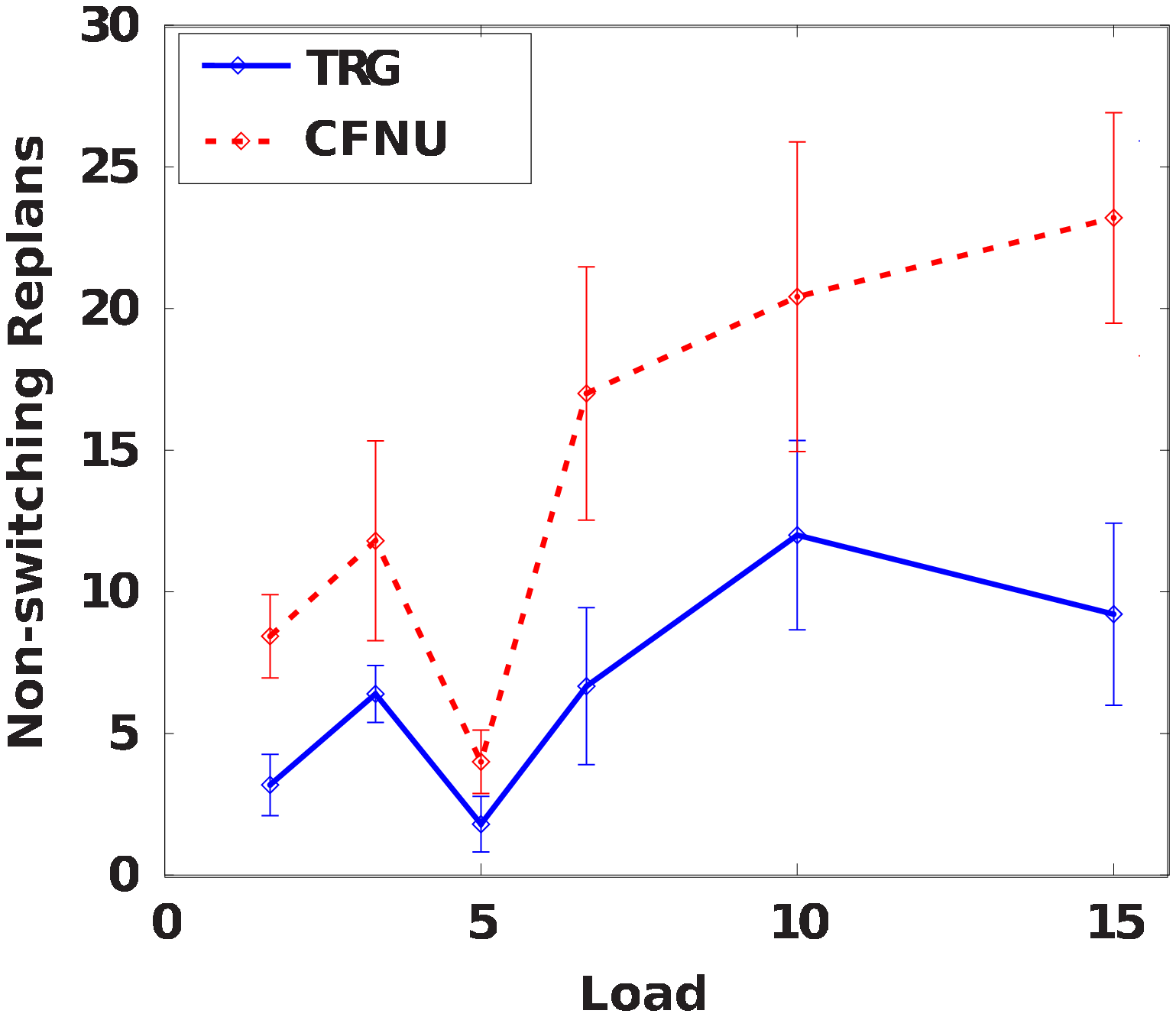}\\
{\small{(a)}} & {\small{(b)}} \\
\includegraphics[width=0.45\linewidth, height=1.4in]{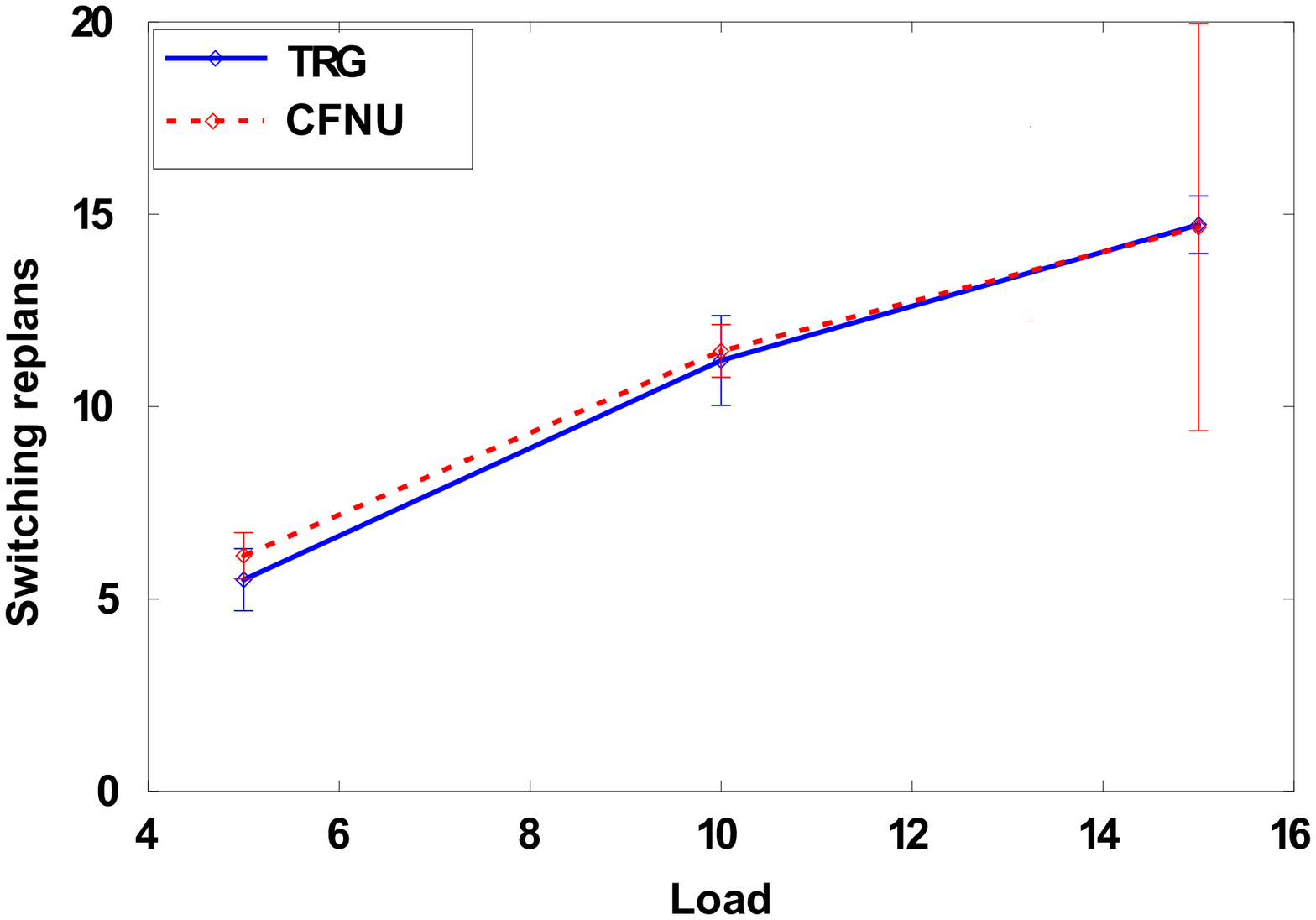}&
\includegraphics[width=0.45\linewidth, height=1.4in]{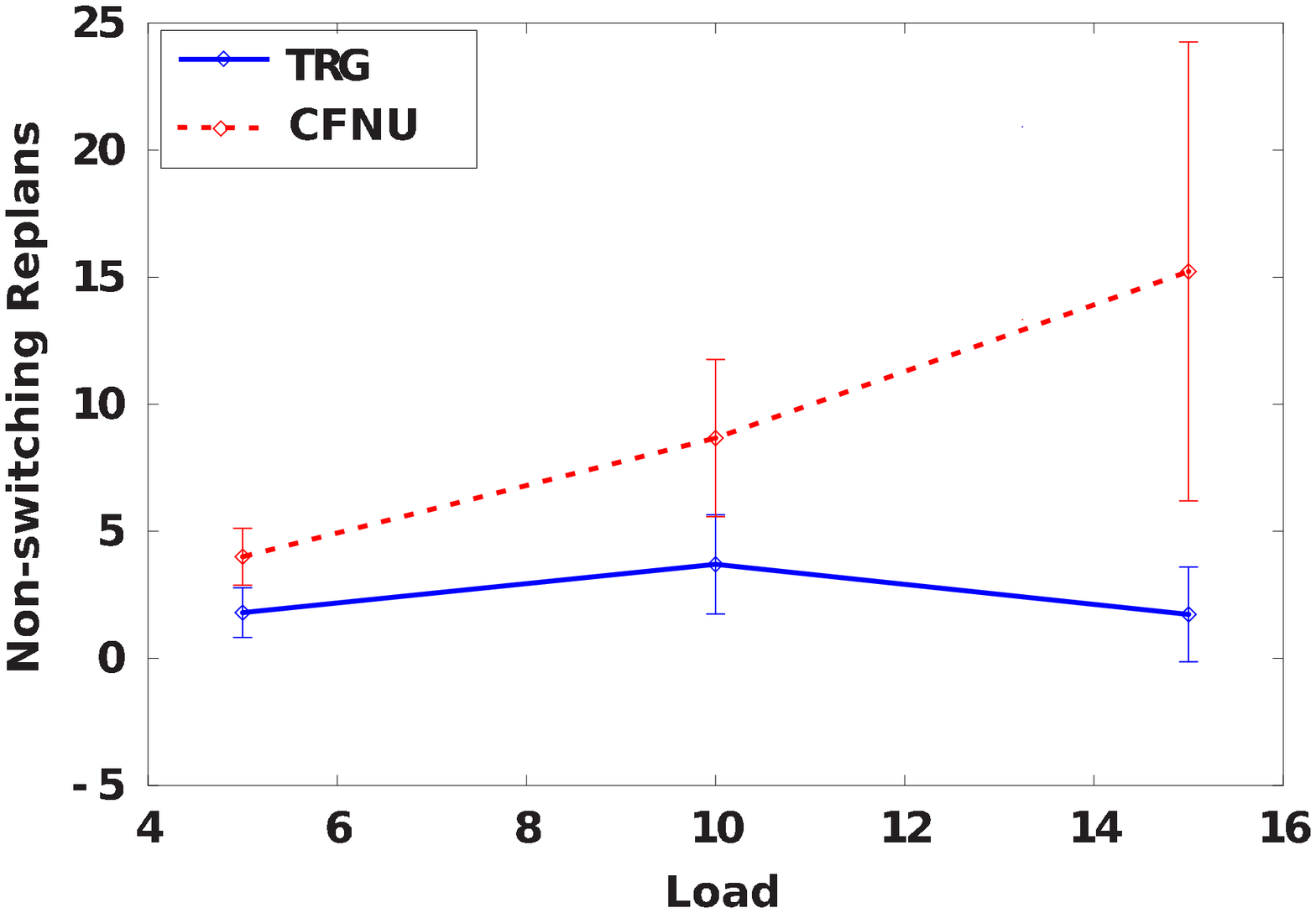} \\
{\small{(c)}} & {\small{(d)}} \\
\end{tabular}
\caption{Average number of replans by each robot which (a) results in
  a task switch, and (b) does not result in a task switch, using the
  TRG and CFNU approaches for different values of $T$, $R$ and
  $Vis$. Values on the x-axis are arranged in increasing orders for
  average robot task load $L$. (c) Switching replans single robot case 
  (d) Non-switching replans single robot case }
\label{fig_num_switches}
\end{figure}

To verify our approach, we have run several experiments on the Webots robot simulator using a model of the Coroware Corobot robot, as well as on physical Corobot robots. The
Corobot is a 4-wheel $30$cm $\times$ $30$cm, skid-steer robot which we control in a differential wheel manner. It is equipped with an indoor Stargazer localization system \cite{stargazer} which provides 2D location and heading, Hokuyo laser distance sensor which provides a $270\degree$ range up to $5$m in distance at a resolution of $\frac{1}{3}\degree$, and wireless communications.  The on-board computer has an 1.6GHz Intel Atom 330 processor with 4GB of RAM.  The simulator provides a Gaussian distributed noise to sensor and actuator motions. The environments have a size of $10 \times 10$ m$^2$; obstacles
were placed at different locations within the environment. The number of tasks ($T$) was varied over $5, 10$ and $15$. The task locations were selected in a way that if there were a single robot in the environment following a closest first task selection strategy, it would cause the robot to switch tasks for $50\%$ of the replans it does.  When there are multiple robots, one of them is selected arbitrarily and placed at a location that would effect the aforementioned $50\%$ task switching for replans. Any other robots are placed randomly in the environment, while keeping an even distribution.  For the different settings we have used, the robot positions are the same for all runs in an environment. The number of robots ($R$) used were $1$ and $3$ . For each task a certain number of robots had to visit it to complete it; the variable for visits per task ($Vis$) was varied over $1$, $2$, and $3$. To present our results in a concise manner, we have used a parameter representing each robot's average task load, given by $L =\frac{Vis \times T}{R}$. The different combinations of $T, R$ and $Vis$, and corresponding values of $L$ used for our experiments are given in Table \ref{tab:loadMapping}. User-defined constants were set to $\gamma=0.8$ (Discount rate of MDP in Equation \ref{eq:mdp}) and $\Gamma=1.5$ (PLL parameter in Equation \ref{eq:pll}), unless otherwise stated. Probability values used in the HMM were similar to those given in Figure \ref{fig:dbn}, which were based on average times for a robot for avoiding static and mobile obstacles within the environments used for the experiments. All results were averaged over $10$ simulation runs. We have compared the quality of task ordering performed by our proposed technique with an approach where the task ordering is done using CFNU~\cite{Woosley-FLAIRS-2013} - each robot selects the task that has the least cost on its TRG from its current location, without modeling uncertainties in the inter-task paths or updating TRG edge availabilities. The CFNU algorithm switches tasks when another task is closer to it than the currently selected task, as measured using straight line distance. To enable comparison, both task ordering approaches use the same underlying path planner and multi-robot coordination mechanism, when necessary.  We have reported three metrics for quantifying the performance of the algorithms - the distances traveled by the robots, the number of replans (with and without task switches) and the times (planning and locomotion) taken to visit all tasks. Also, to understand only the performance of our MDP-based task selection method, we have reported the  metrics separately for single robot scenarios, where $|R| = |Vis| = 1$. That is,there is only $1$ robot that has to visit all tasks once, and edge availability is affected only by previously unknown obstacles; effects due to communication uncertainties of {\em TaskComplete} messages sent by other robots do not arise.

\begin{figure}
\centering
\begin{tabular}{c c}
\includegraphics[width=0.45\linewidth]{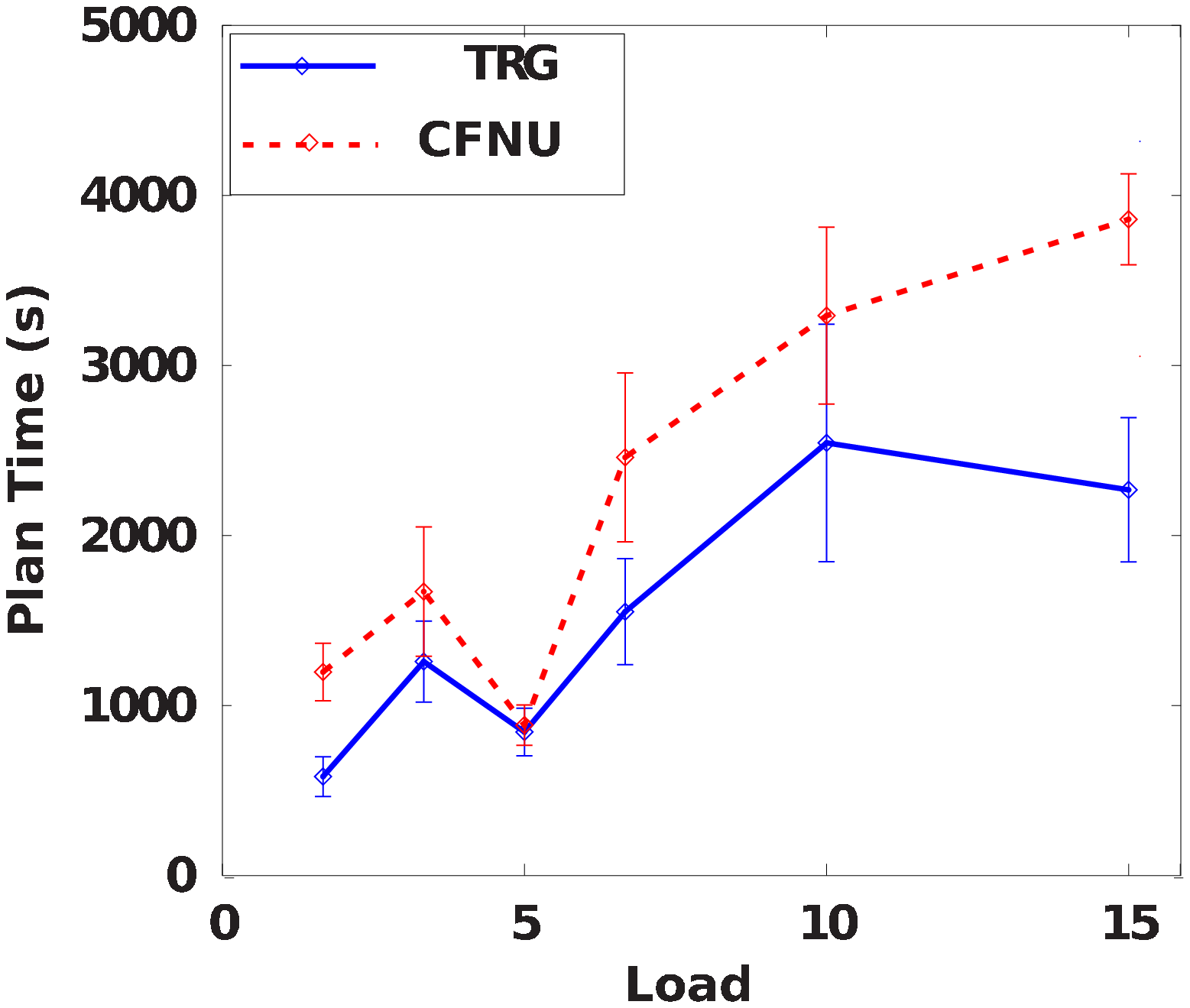} &
\includegraphics[width=0.45\linewidth]{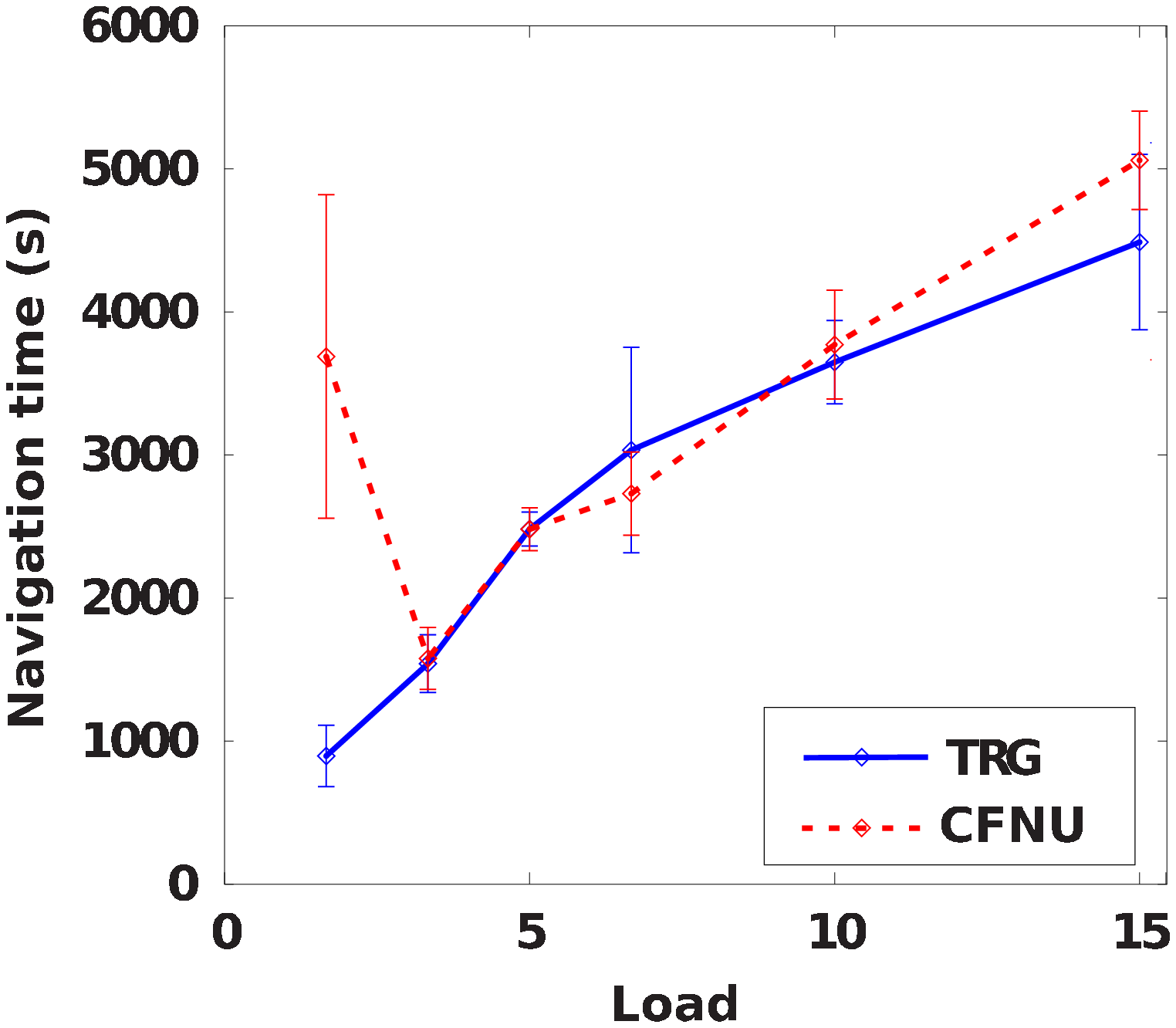} \\ 
{\small{(a)}} & {\small{(b)}} \\
\includegraphics[width=0.45\linewidth, height=1.4in]{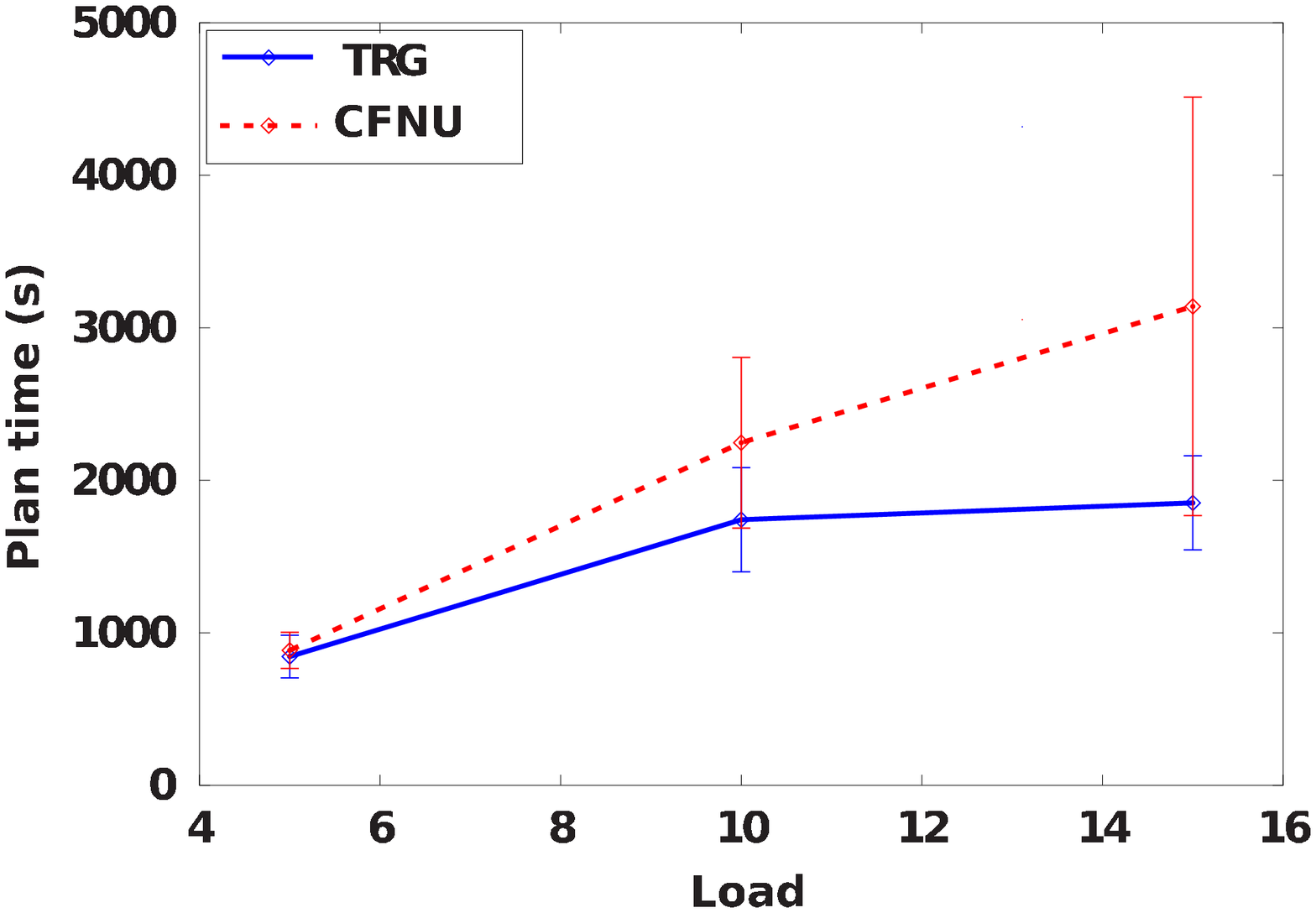} &
\includegraphics[width=0.45\linewidth, height=1.4in]{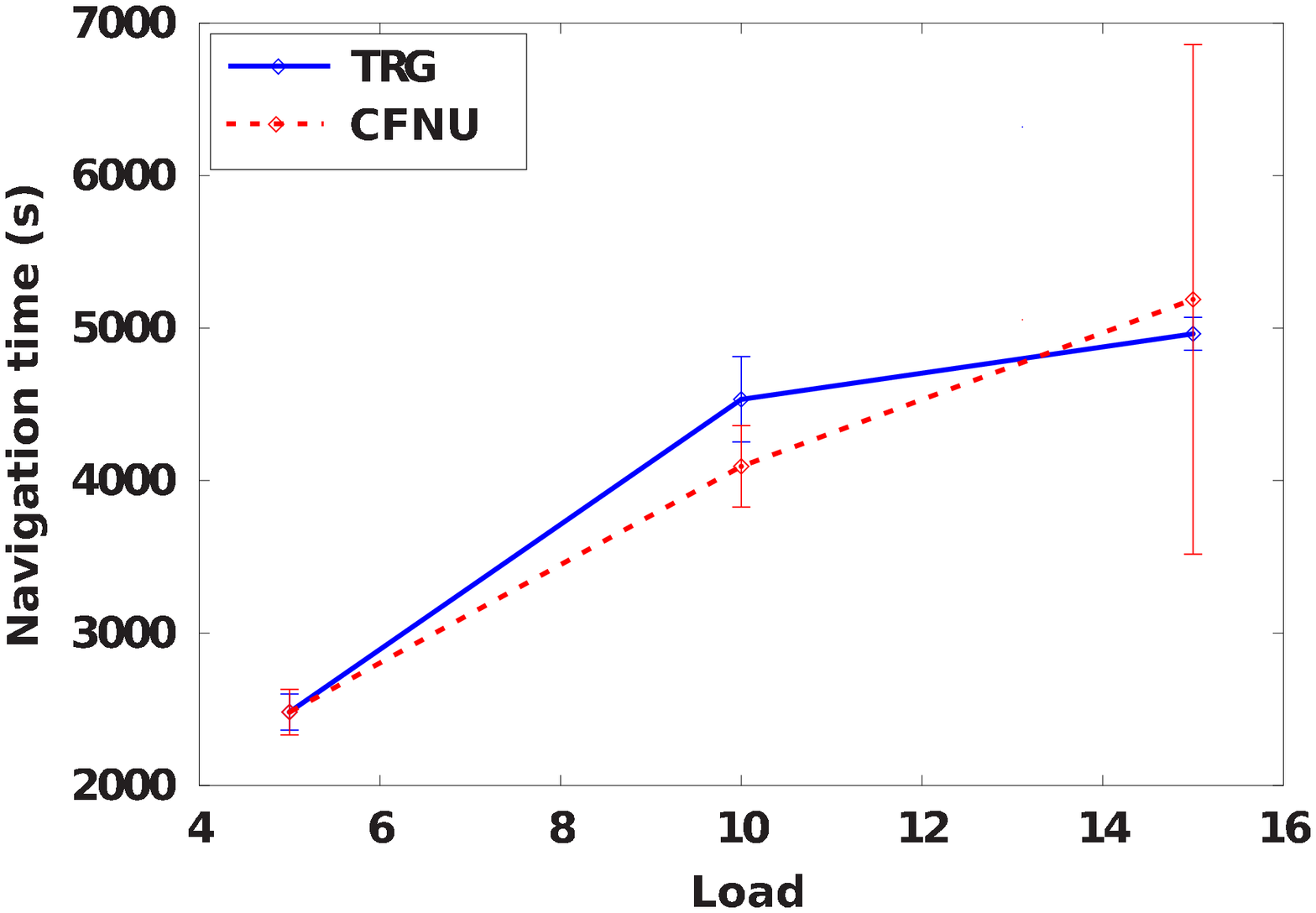} \\
{\small{(c)}} & {\small{(d)}}\\
\end{tabular}
\caption{Average times taken by TRG and CFNU approaches. Top row is for multi-robot scenarios, bottom row is for single robot scenarios.  The left column shows time for planning paths and solving the MDP for the TOP-U problem. The right column shows time taken traveling between tasks and performing collision avoidance.  Experiments performed for different values of $T$, $T$ and $Vis$. Values on the x-axis are arranged in increasing orders for average robot task load $L$.
}
\label{fig_num_time}
\end{figure}

Figures~\ref{fig_num_switches}(a) and (b) show the number of replans made by each robot, resulting in and not resulting in task switches respectively for the two algorithms for the different load values shown in Table~\ref{tab:loadMapping}. We observe that, on average, the TRG-based approach results in $40\%$ less replanning and $61\%$ less task switching than the CFNU approach. The reduced planning and task switching by the TRG-based algorithm can be attributed to its ability to reason more efficiently about task availabilities using its costs and beliefs about paths in the MDP based approach, along with real-time sensor data incorporated into its decisions using the HMM. In contrast, the CFNU approach uses only Euclidean distances to select tasks and consequently performs poorly. In Figures~\ref{fig_num_switches}(c) and (d), the number of replans resulting in and not resulting in tasks switches are shown for the single robot cases. We see that the switching replans are the same for both approaches because there are no other robots in the environment which could complete the task before the robot reaches it first. We also see that the TRG performs fewer non-switching replans.

\begin{figure}[htb]
\centering
\begin{tabular}{cc}
\includegraphics[width=0.45\linewidth]{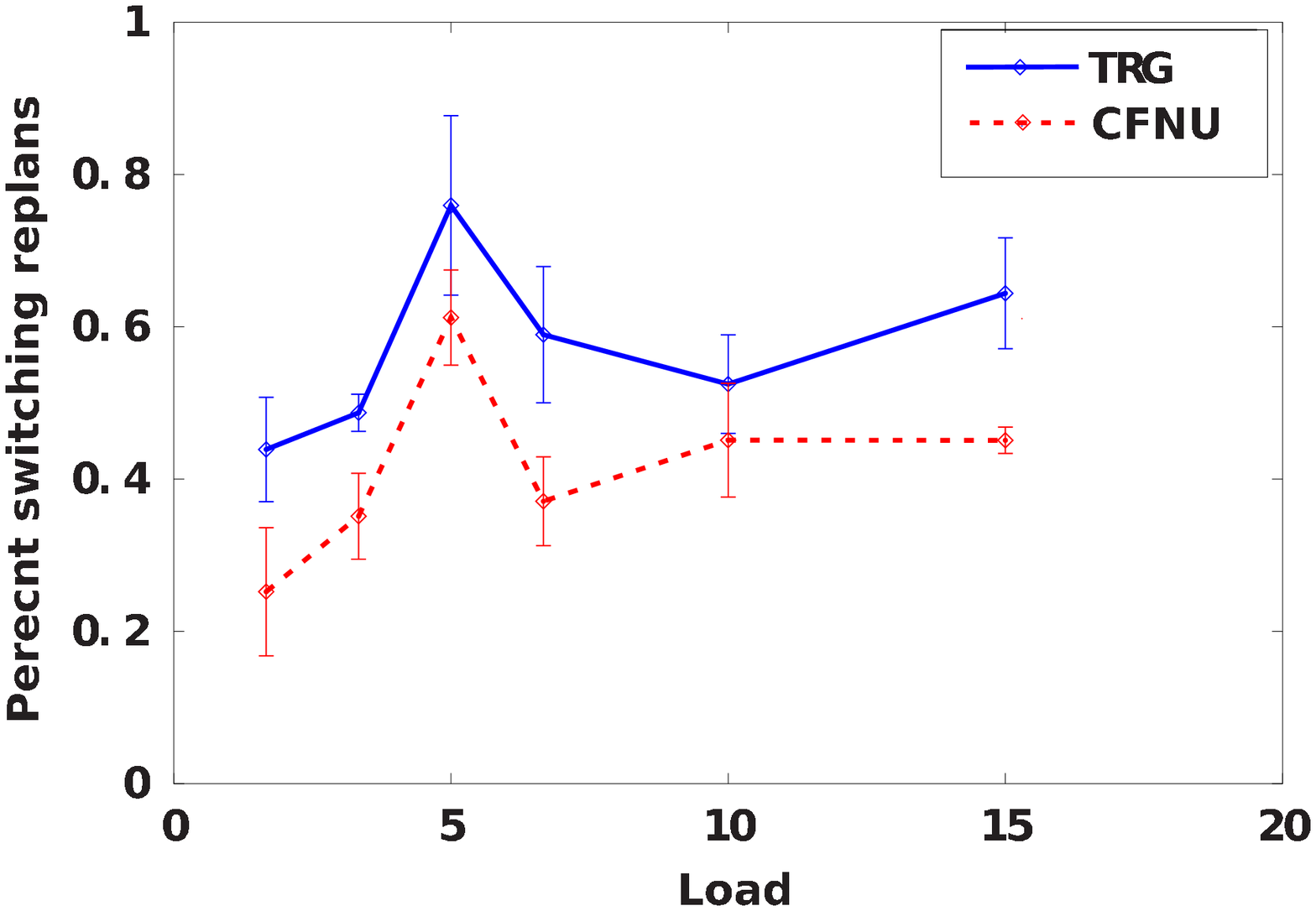} &
\includegraphics[width=0.45\linewidth]{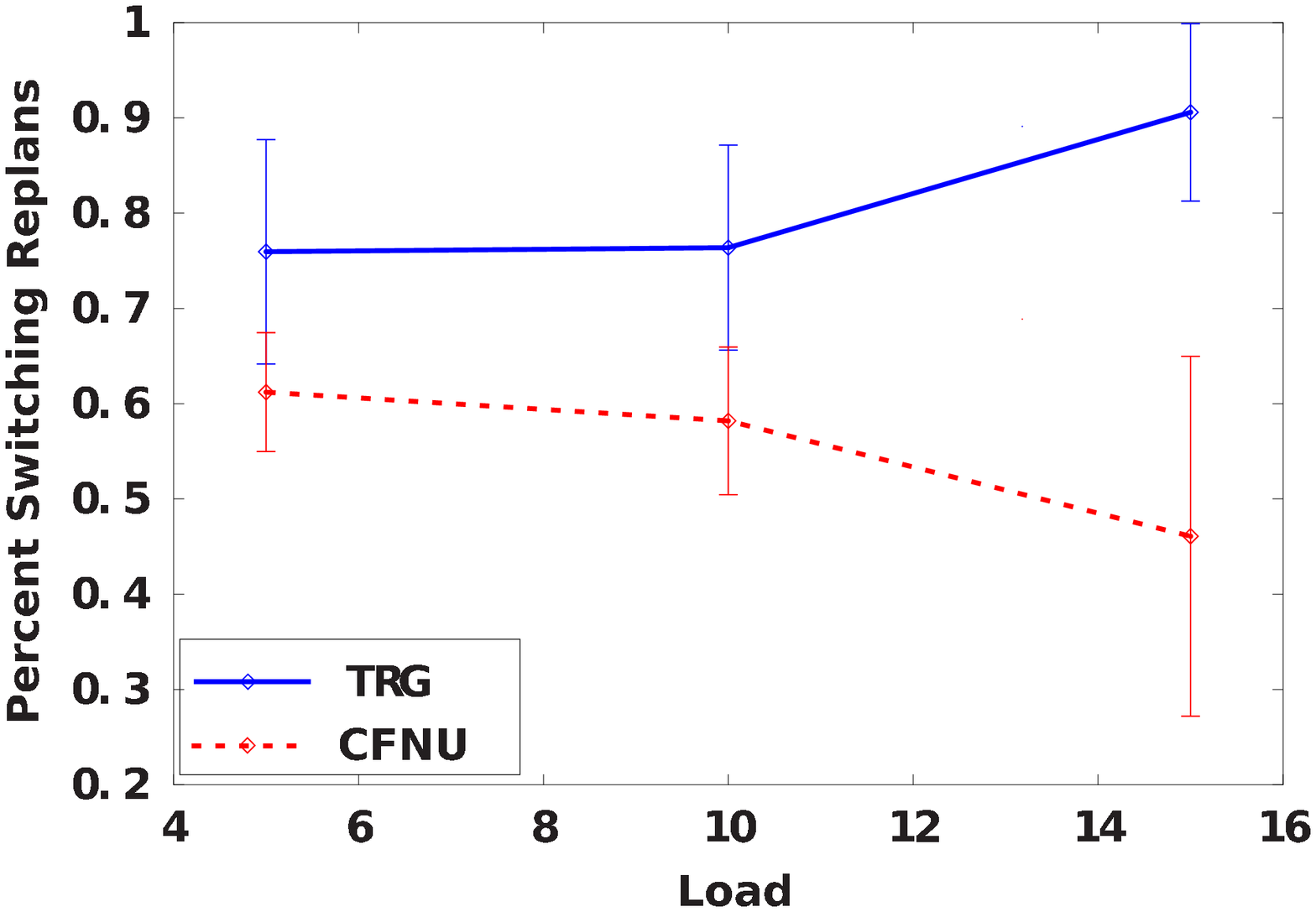} \\
(a) & (b) \\
\end{tabular}
\caption{Percent of replans resulting in a task switch (a) all cases (b) single robot case}
\label{fig:percentReplans}
\end{figure}

In Figures \ref{fig_num_time}(a) and (b), we show the average time taken for both approaches, this includes path planning and task ordering times, and, locomotion times for each robot, which includes time taken to resolve inter-robot path conflicts using Algorithm~\ref{algo_collision_avoidance}.  The TRG-based approach takes much less time for both planning and navigation compared to the CFNU approach.  In this case, the TRG-based approach requires up to $51\%$ less planning time, and, up to $75\%$ lower locomotion and coordination times than the CFNU approach. This is because the TRG-based approach accounts for both the known obstacles between tasks and the likelihood that the task will become unavailable. The CFNU approach behaves myopically and selects the closest task to visit, which could be on the other side of a large obstacle and require considerable planning and locomotion times to reach. In contrast, the TRG-based approach uses the robot’s perception of the environment to weight the path costs to tasks with the corresponding path belief to reduce the overall path costs. Note that when the number of tasks is small, or the average task load per robot is close to 1, both algorithms have comparable performance for all three metrics as each robot has to visit only one task and there is no task ordering required. Figures~\ref{fig_num_time}(c) and (d) show the average time taken by TRG and CFNU approaches in the single robot case as the average robot load increases. We can see that our TRG approach takes less planning time compared to the CFNU approach.

We also observe that as the average task load $L$ of the robots increase (from left to right on the x-axis), the distances traveled by the robots increases. The robots using the TRG-based approach travel  distances between $-16\%$ (more) to $32\%$ (less) than the CFNU approach, with an average improvement of about $6\%$ (less distance traveled) across all experiments. The TRG-based approach sometimes travels a small amount more than the CFNU approach when the robot decides to abandon its current task for another task. In that case, the CFNU approach will switch as soon as the other task becomes closer, which in some cases is the best decision, where as the TRG-based approach will continue to follow its previous task even though another task is closer. In some cases this might be the best thing to do because the closer task might be on the other side of a wall that the robot has yet to discover and actually require more distance to explore enough of the wall to realize that the task is no longer the closest; TRG performs better in such scnearios. In some of the environments with a larger load value, the TRG-based approach starts to perform better.

We also tested our proposed algorithm on physical Corobot robots.  We used the environment shown in Table~\ref{tab:hardwareResults} (left) where the white dots represent the tasks that the robot must visit.  This environment was also designed such that a CFNU closest-task-first algorithm would switch $50\%$ of the time.  We tested for a fixed value of $\Gamma =5.0$, and tested for $|R| = \{1,2\}$, $Vis = \{1,2\}$, and $|T|=5$.  Results were averaged over three runs, and are shown in Table~\ref{tab:hardwareResults} (right). In both environments, the robots were able to navigate and visit the tasks with the required number of visits.  In comparing these results to those shown in Figure~\ref{fig_num_switches}, we can see that the hardware performed fairly similar.  For example, Table~\ref{tab:hardwareResults} shows that the five tasks, one robot, one visit environment had four switching replans, which is within the margin of error for the switching replans.

\omitit{
\begin{figure}
\includegraphics[width=1.8in]{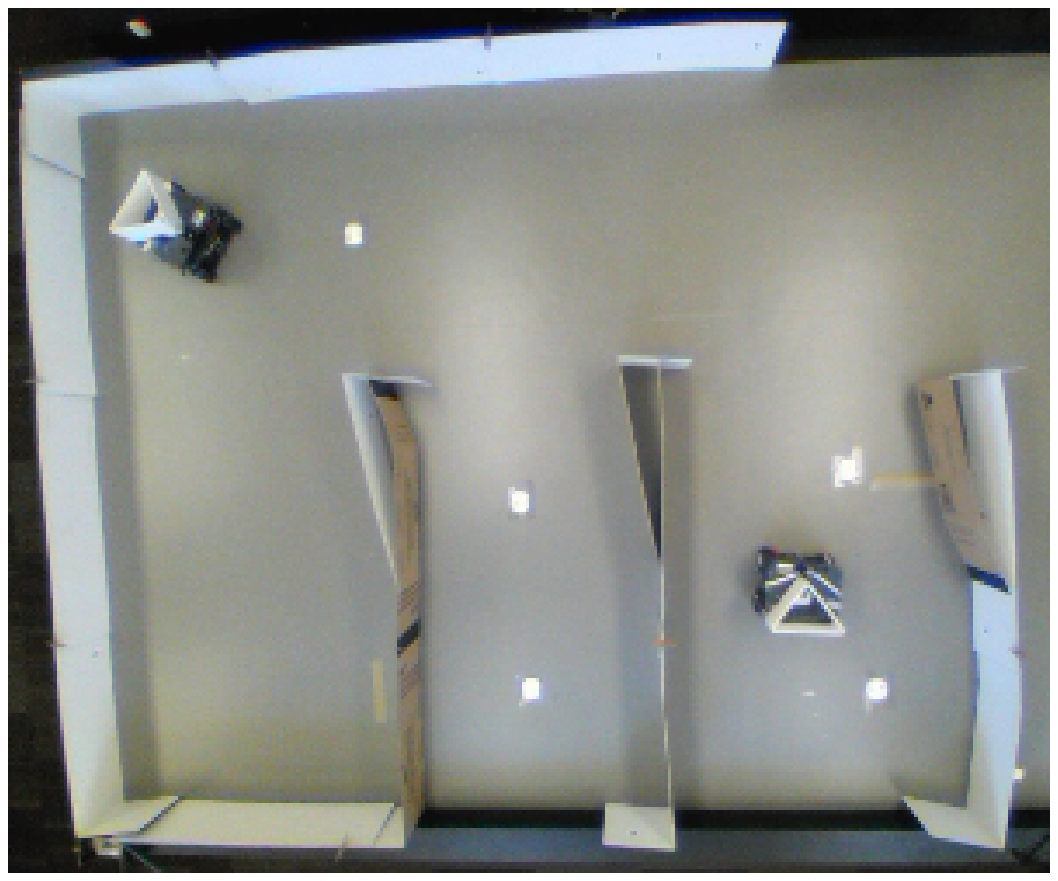} 
\caption{Overhead view of environment used for testing with $2$ Corobot robots and $5$ tasks; white dots represent task locations}
\label{fig:hardwareEnv}
\end{figure}
}

\begin{table*}
\centering
\caption{Overhead view of environment used for testing with $2$ Corobot robots and $5$ tasks; white dots represent task locations (left) and experimental results}
\label{tab:hardwareResults}
\begin{tabular}{c c c l l l l l}
\multirow{5}{*}{\hspace*{-0.1in}\includegraphics[width=1.3in, height=1.0in]{figs/hardware_env_v2.eps} }&\multicolumn{7}{c}{}\\
& \multicolumn{7}{c}{}\\
\hhline{~-------}
&\textbf{Rob.} & \textbf{Vis.} & \textbf{Dist. Traveled (cm)} & \textbf{\# Switches} & \textbf{\# Non-switches} & \textbf{Plan Time (s)} & \textbf{Navigation Time (s)} \\ \hhline{~-------}
&$1$ & $1$ & $2102 \pm (349.94)$ & $4 \pm (0)$ & $0.67 \pm (1.15)$ & $613.96 \pm (144.54)$ & $1415.96 \pm (238.87)$ \\ \hhline{~-------}
&$2$ & $1$ & $1130.45 \pm (86.92)$ & $2.83 \pm (0.58)$ & $2.5 \pm (1.32)$ & $835.6 \pm (269.11)$ & $746.86 \pm (48.66)$ \\ \hhline{~-------}
&$2$ & $2$ & $6118.95 \pm (2689.5)$ & $4.5 \pm (1.5)$ & $3.17 \pm (2.52)$ & $760 \pm (372.83)$ & $3921.43 \pm (1704.98)$ \\ \hhline{~-------}
\end{tabular}
\end{table*}

\section{Conclusions and Future work}
In this paper, we introduced the TOP-U problem where robots have to determine the order to visit a set of task locations when the path costs between a pair of tasks can vary dynamically as the robot discovers obstacles while navigating between tasks. We proposed a data structure called a task reachability graph (TRG) along with techniques
to integrate the task ordering with uncertainty in path costs and availabilities. Our analytical results show that our proposed algorithm is optimal and complete. The algorithm's performance was also  evaluated extensively through experiments and was shown to result in reduced time and fewer computations (less task switching) as compared to an algorithm that does not consider uncertainty in path costs and task availabilities. As future work, we are considering analyzing situations with stricter constraints such as a partial ordering over the task set. In the present work, multi-robot coordination is handled using a light-weight coordination mechanism where all robots, except one, stop. We are investigating techniques to integrate tighter, but fast, coordination approaches to improve the coordination of robots, such as exchanging path plans and calculating a plan in the joint configuration space, only when robots get within close proximity of each other. 

\bibliographystyle{IEEEtran}
\bibliography{refs}
\end{document}